\documentclass{article}

\usepackage[preprint]{my_neurips}

\usepackage[utf8]{inputenc} 
\usepackage[T1]{fontenc}             
\usepackage{url}            
\usepackage{booktabs}      
\usepackage{amsfonts}      
\usepackage{nicefrac}       
\usepackage{microtype}      
\usepackage{xcolor}        
\usepackage{amsfonts}
\usepackage{amsmath}
\usepackage{amsthm}
\usepackage{amssymb}
\usepackage{dsfont}
\usepackage{multirow}
\usepackage{xcolor}
\usepackage{color}
\usepackage{graphicx}

\usepackage{algorithm}
\usepackage[noend]{algpseudocode}
\usepackage{verbatim}
\usepackage{xspace} 

\usepackage{enumerate}
\usepackage{enumitem}

\providecommand{\lin}[1]{\ensuremath{\left\langle #1 \right\rangle}}

  \providecommand{\R}{\mathbb{R}} 
  
  \DeclareMathOperator{\E}{{\mathbb E}}

  \renewcommand{\aa}{\mathbf{a}}
  \providecommand{\bb}{\mathbf{b}}

  \providecommand{\ww}{\mathbf{w}}
  \providecommand{\xx}{\mathbf{x}}
  \providecommand{\yy}{\mathbf{y}}

  \providecommand{\cN}{\mathcal{N}}
  \providecommand{\cO}{\mathcal{O}}

  \providecommand{\cX}{\mathcal{X}}

  \providecommand{\cZ}{\mathcal{Z}}

  \usepackage{bm}

\RequirePackage[colorinlistoftodos,bordercolor=orange,backgroundcolor=orange!20,linecolor=orange,textsize=scriptsize]{todonotes}
\providecommand{\mycomment}[3]{\todo[caption={},color=#3!20,inline]{\textbf{#1: }#2}}%
\providecommand{\myinlinecomment}[3]{%
  {\color{#1}#2: #3}}%
\newcommand\commenter[2]%
{%
  \expandafter\newcommand\csname i#1\endcsname[1]{\myinlinecomment{#2}{#1}{##1}}
  \expandafter\newcommand\csname #1\endcsname[1]{\mycomment{#1}{##1}{#2}}
}
  
\newtheorem{lemma}{Lemma}
\newtheorem{corollary}[lemma]{Corollary}

\newtheorem{definition}{Definition}
\newtheorem{remark}[lemma]{Remark}

\newtheorem{theorem}[lemma]{Theorem}
\newtheorem{example}[lemma]{Example}
\usepackage[colorlinks=true,linkcolor=blue]{hyperref} 
\usepackage[capitalize,noabbrev]{cleveref}

\usepackage{url}

\renewcommand{\epsilon}{\varepsilon}

\usepackage{tcolorbox}

\newtcbox{\comparison}{on line,
  colframe=blue,colback=white,
  boxrule=0.5pt,arc=4pt,boxsep=0pt,left=6pt,right=6pt,top=6pt,bottom=6pt}
  
\usepackage[font=small]{caption}

\title{Adaptive SGD with Polyak stepsize and Line-search: Robust Convergence and Variance Reduction}

\author{
  Xiaowen Jiang  \\
  CISPA\thanks{CISPA Helmholtz Center for Information Security}  \\
  Saarbrücken, Germany\\
  \texttt{xiaowen.jiang@cispa.de}\\
  \And
  Sebastian U. Stich \\ 
  CISPA\footnotemark[1] \\ 
  Saarbrücken, Germany\\
  \texttt{stich@cispa.de} \\
}
\commenter{seb}{red}
\commenter{xiaowen}{blue}
\begin{document}

\maketitle

\begin{abstract}
The recently proposed stochastic Polyak stepsize (SPS) and stochastic line-search (SLS) for SGD have shown remarkable effectiveness when training over-parameterized models. However, in non-interpolation settings, both algorithms only guarantee convergence to a neighborhood of a solution which may result in a worse output than the initial guess. While artificially decreasing the adaptive stepsize has been proposed to address this issue (\citet{decsps}), this approach results in slower convergence rates for convex and over-parameterized models. In this work, we make two contributions:
Firstly, we propose two new variants of SPS and SLS, called AdaSPS and AdaSLS, which guarantee convergence in non-interpolation settings and maintain sub-linear and linear convergence rates for convex and strongly convex functions when training over-parameterized models. 
AdaSLS requires no knowledge of problem-dependent parameters, and AdaSPS requires only a lower bound of the optimal function value as input.
Secondly, we equip AdaSPS and AdaSLS with a novel variance reduction technique and obtain algorithms that require $\smash{\widetilde{\mathcal{O}}}(n+1/\epsilon)$ gradient evaluations to achieve an $\mathcal{O}(\epsilon)$-suboptimality for convex functions, which improves upon the slower $\mathcal{O}(1/\epsilon^2)$ rates of AdaSPS and AdaSLS without variance reduction in the non-interpolation regimes. Moreover, our result  matches the fast rates of AdaSVRG but removes the inner-outer-loop structure, which is easier to implement and analyze.
Finally, numerical experiments on synthetic and real datasets validate our theory and demonstrate the effectiveness and robustness of our algorithms.
\end{abstract}

\section{Introduction}
Stochastic Gradient Descent (SGD) \citep{sgd_robbin} and its variants \citep{large_opt_book} are among the most preferred optimization algorithms for training modern machine learning models. These first-order methods only compute the stochastic gradients in each iteration, which is often more efficient than computing a full batch gradient. However, the performance of SGD is highly sensitive to the choice of the stepsize. One common approach is to use a fixed stepsize schedule, for instance, to keep it constant or to decrease it over time. Unfortunately, the theoretically optimal schedules are disparate across different function classes \citep{bubek_convex}, and usually depend on problem parameters that are often unavailable, such as the Lipschitz constant of the gradient. As a result, a heavy tuning of the stepsize parameter is required, which is typically expensive in practice. 

Instead of fixing the stepsize evolution, adaptive SGD methods adjust the stepsize on the fly \citep{AdaGrad,adam}. These algorithms often require less hyper-parameter tuning and still enjoy competitive performance in practice. Stochastic Polyak Stepsize (SPS)~\cite{sps,alig,polyak-book} is one of such recent advances. It has received rapid growing interest due to two factors: 1) the only possibly unknown parameter is the individual optimal function value which is often available in many machine learning applications and 2) its adaptivity utilizes the local curvature and smoothness information allowing the algorithm to accelerate and converge quickly when training over-parametrized models. However, when the optimal function value is inexact or the problems are non-interpolated, SPS cannot converge to the solution~\cite{sps}. \citet{decsps} address this issue by applying an artificially decreasing rule to SPS and the resulting algorithm DecSPS is able to converge as quickly as SGD with the optimal fixed decreasing stepsize schedule for convex problems. Furthermore, only a lower bound of the function value is needed for running DecSPS. Despite its convergent property in non-interpolated settings, the asymptotic convergence rates are much slower than SPS in interpolation regimes. 

Stochastic Line-Search (SLS) \citep{sls} is another adaptive stepsize that offers exceptional performance when training over-parametrized models. Compared with SPS, the knowledge of the optimal function value is not required, but at the cost of additional function value evaluations per iteration. However, SLS cannot converge in non-interpolated settings as well. 

There are no current methods that are \emph{robust}, in the sense that they can automatically adapt to the optimization setting (interpolation vs.\ non-interpolation). An ideal adaptive stepsize should enjoy robust convergence. This is because, in many real-world scenarios, it can be challenging to ascertain whether a model is effectively interpolating the data or not. Consider the setting of federated learning~\cite{FL} where millions of clients jointly train a machine learning model on their mobile devices, which usually cannot support huge-scale models.  Due to the fact that each client's data is stored locally, it becomes impractical to check the interpolation condition, i.e.\ to compare the optimal function value of the entire dataset with that of each individual data point. In this context, directly using SPS or DecSPS may result in slow or non-existent convergence. Thus, seeking a robust and parameter-free adaptive stepsize can be of great convenience for the users in practice.

Although SPS and SLS have become popular these days, there exist no current algorithms that successfully combine variance-reduction techniques \citep{VR_zhang,sarah,saga,loopless_svrg,sag,page} with these stepsizes, with a theoretical guarantee. Indeed, \citet{ada-svrg} provide a counter-example where classical line-search methods fail in the VR setting. As such, it is still unclear whether we can accelerate SGD with stepsizes from Polyak and line-search family in non-interpolated settings.

In this work, we aim to provide solutions to the aforementioned two issues, that is: 1) SPS, SLS, and DecSPS do not have robust convergence and 2) There is no theoretically sound algorithm that successfully accelerates SGD with Polyak and line-search type stepsizes with variance-reduction.

\subsection{Main contributions}
We summarize our main contributions as follows (Table \ref{tab:summary} summarizes the  complexity results established in this paper):
\begin{itemize}[leftmargin=*]
    \item In Section \ref{sec:adasps}, we propose the first adaptive methods that simultaneously achieve the best-known asymptotic rates in both strongly-convex or convex and interpolation or non-interpolation settings except for the case when we have strongly-convexity and non-interpolation. The first method called AdaSPS, a variant of SPS, requires only a lower bound of the optimal function value as input (similar to DecSPS) while AdaSLS, the second method based on SLS, is parameter-free. In the non-interpolated setting, we prove for both algorithms an $\cO(1/\epsilon^2)$ convergence rate for convex functions which matches the classical DecSPS and AdaGrad~\cite{AdaGrad} results, whereas SPS and SLS cannot converge in this case. In the interpolated regime, we establish fast $\cO(\log(1/\epsilon))$ and $\cO(1/\epsilon)$ rates under strong convexity and convexity conditions respectively, without knowledge of any problem-dependent parameters. In contrast, DecSPS converges at the slower $\cO(1/\epsilon^2)$ rate and for AdaGrad, the Lipschitz constant is needed to set its stepsize \cite{linear-AdaGrad}. 
    
    \item In Section \ref{sec:adasvrps}, 
    we design the first variance-reduced methods that can use Polyak stepsizes or line-search.
    We prove that, to reach an $\epsilon$-accuracy, the total number of gradient evaluations required in expectation is $\smash{\widetilde{\mathcal{O}}}(n+1/\epsilon)$ for convex functions which matches the rate of AdaSVRG \cite{ada-svrg}.
    With our newly proposed decreasing probability strategy, the artificially designed multi-stage inner-outer-loop structure is not needed, which makes our methods easier to analyze.

    Our novel VR-framework is based on proxy function sequences and can recover the standard VR methods~\cite{VR_zhang} as a special case.  We believe that this technique can be of independent interest to the optimization community and may motivate more personalized VR techniques in the future.
    \item 
    In Section \ref{sec:ex}, we verify our theoretical results by conducting numerical experiments.
\end{itemize}

\begin{table*}[ht!]
\resizebox{\textwidth}{!}
{\begin{minipage}{1.35\textwidth}
\centering
\begin{tabular}{@{}ccccccc@{}}
\toprule
\multirow{2}{*}{\textbf{Stepsize}} &
 \multicolumn{3}{c}{\textbf{Interpolation}} & \multicolumn{3}{c}{\textbf{Non-interpolation}}  \\ 
 \cmidrule(lr){2-4}\cmidrule(lr){5-7} &
\multicolumn{1}{c}{strongly-convex} & \multicolumn{1}{c}{convex} & \multicolumn{1}{c}{required input} & \multicolumn{1}{c}{strongly-convex}& \multicolumn{1}{c}{convex\footnote{The assumption of bounded iterates is also required except for $\text{SPS}$ and SLS.}} & \multicolumn{1}{c}{required input} \\
\midrule
$\text{SPS}/\text{SPS}_{\max}$ \cite{sps} & $\cO(\log(\frac{1}{\epsilon}))$ & $\cO(\frac{1}{\epsilon})$ & $f_{i_t}^\star$ & $\epsilon\ge\Omega(\sigma_{f,B}^2)$& $\epsilon\ge\Omega(\sigma_{f,B}^2)$ & $f_{i_t}^\star$ \\
\rule{0pt}{3ex} 
$\text{SLS}$ \cite{sls} & $\cO(\log(\frac{1}{\epsilon}))$ & $\cO(\frac{1}{\epsilon})$ & None & $\epsilon\ge\Omega(\sigma_{f,B}^2)$& $\epsilon\ge\Omega(\sigma_{f,B}^2)$ & None \\ 
\rule{0pt}{3ex} 
$\text{DecSPS}$ \cite{decsps} & $\cO(\frac{1}{\epsilon^2})$ & $\cO(\frac{1}{\epsilon^2})$ & $\ell_{i_t}^\star$ & $\cO(\frac{1}{\epsilon^2})$ & $\cO(\frac{1}{\epsilon^2})$  & $\ell_{i_t}^\star$ \vspace{0.1cm}\\ 
\hline
\rule{0pt}{3ex} 
$\text{AdaSPS}$ (this work) & $\cO(\log(\frac{1}{\epsilon}))$ & $\cO(\frac{1}{\epsilon})$ &$f_{i_t}^\star$ & $\cO(\frac{1}{\epsilon^2})$ & $\cO(\frac{1}{\epsilon^2})$  & $\ell_{i_t}^\star$\\
\rule{0pt}{3ex} 
$\text{AdaSLS}$ (this work) & $\cO(\log(\frac{1}{\epsilon}))$ & $\cO(\frac{1}{\epsilon})$ & None & $\cO(\frac{1}{\epsilon^2})$ & $\cO(\frac{1}{\epsilon^2})$  & None  \\
\bottomrule
\end{tabular}
\end{minipage}}
\caption{\small{Summary of convergence behaviors of the considered adaptive stepsizes for smooth functions. For $\text{SPS}/\text{SPS}_{\text{max}}$ and SLS in non-interpolation settings, $\Omega(\cdot)$ indicates the size of the neighborhood that they can converge to. In the other cases, the $\cO(\cdot)$ complexity provides the total number of gradient evaluations required for each algorithm to reach an $\cO(\epsilon)$ suboptimality. For convex functions, the suboptimality is defined as $\E[f(\Bar{\xx}_T)-f^\star]$ and for strongly convex functions, the suboptimality is defined as $\E[||\xx_T-\xx^\star||^2]$. }} 
\label{tab:summary}
\end{table*}

\subsection{More related work}
Line-search procedure has been successfully applied to accelerate large-scale machine learning training. Following \cite{sls}, \citet{relax_ls} propose to relax the condition of monotone decrease of objective function for training over-parameterized models. \citet{multidimls}  extends backtracking line-search to a multidimensional variant which provides a better diagonal preconditioners. In recent years, adaptive stepsizes from the AdaGrad family have become widespread and are particularly successful when training deep neural networks. Plenty of contributions have been made to analyze variants of AdaGrad for different classes of functions \cite{AdaGrad,adanorm-2010,adanorm-2015,wald,linear-AdaGrad}, among which \citet{adagrad+ls} first propose to use line-search to set the stepsize for AdaGrad to enhance its practical performance. More recently, variance reduction has successfully been applied to AdaGrad stepsize and faster convergence rates have been established for convex and non-convex functions \cite{ada-svrg,svrg-nonconvex}.

Another promising direction is the Polyak stepsize (PS) \cite{polyak-book} originally designed as a subgradient method for solving non-smooth convex problems. \citet{hazan_polyak} show that PS indeed gives simultaneously the optimal convergence result for a more general class of convex functions. \citet{increment-polyak} propose several variants of PS as incremental subgradient methods and they also discuss the method of  dynamic estimation of the optimal function value when it is not known. Recently, more effort has been put into extending deterministic PS to the stochastic setting \cite{L4,alig,oberman2019}. However, theoretical guarantees of the algorithms still remain elusive until the emergence of SPS/SPS$_{\max}$ \cite{sps} which both demonstrates empirical success and provides strong theoretical guarantees. Subsequently, further improvements and new variants such as DecSPS \cite{decsps} and SPS with a moving target \cite{spsmt} have been introduced.  A more recent line of work  interprets stochastic Polyak stepsize as a subsampled Newton Raphson method and interesting algorithms have been designed based on the first-order local expansion \cite{spsmt,gower2022cutting} as well as the second-order expansion~\cite{sp2}.

\section{Problem setup and background}
\label{sec:background}
\subsection{Notations}
In this work, we consider solving the finite-sum smooth convex optimization problem:
\begin{equation}
    \min_{\xx \in \R^d} \left[ f(\xx)=\frac{1}{n}\sum_{i=1}^n f_i(\xx) \right] \,.
\label{eq:problem}
\end{equation}
This type of problem appears frequently in the modern machine learning applications \citep{hastie}, where each $f_i(\xx)$ represents the loss of a model on the $i$-th data point parametrized by the parameter $\xx$. Stochastic Gradient Descent (SGD) \citep{sgd_robbin} is one of the most popular methods for solving the problem~\eqref{eq:problem}. At each iteration, SGD takes the  form: 
\begin{equation}
\xx_{t+1} = \xx_t - \eta_t \nabla f_{i_t}(\xx_t)\;,
\label{SGD}
\end{equation}
where $\eta_t$ is the stepsize parameter, 
$i_t\subseteq[n]$ is a random set of size $B$ sampled independently at each iteration $t$ and $\nabla f_{i_t}(\xx)=\frac{1}{B}\sum_{i\in i_t}\nabla f_i (\xx)$ is the minibatch gradient. 

Throughout the paper, we assume that there exists a non-empty set of optimal points $\cX^\star\subset\R^d$, and we use $f^\star$ to denote the optimal value of $f$ at a point $\xx^\star\in\cX^\star$. 
We use $f_{i_t}^\star$ to denote the infimum of minibatch function $f_{i_t}(x)$, i.e.\ $f_{i_t}^\star=\inf_{\xx\in\R^d}\frac{1}{B}\sum_{i\in i_t}f_{i}(\xx)$. We assume that all the individual functions $\{f_i(\xx)\}$ are $L$-smooth. Finally, we denote the optimal objective difference, first introduced in \cite{sps}, by $\sigma_{f,B}^2=f^\star-\E_{i_t}[f_{i_t}^\star]$. The problem~\eqref{eq:problem} is said to be interpolated if $\sigma_{f,1}^2=0$, which implies that $\sigma_{f,B}^2=0$ for all $B\le n$ since $\sigma_{f,B}^2$ is non-increasing w.r.t $B$.
 
\subsection{SGD with stochastic polyak
stepsize} 
 \citet{sps} propose to set $\eta_t$ as: $\eta_t=2\frac{f_{i_t}(\xx_t)-f_{i_t}^\star}{||\nabla f_{i_t}(\xx_t)||^2}$, which is well known as the Stochastic Polyak stepsize (SPS). In addition to SPS, they also propose a bounded variant SPS$_{\max}$ which has the form $\eta_t=\min\bigl\{2\frac{f_{i_t}(\xx_t)-f_{i_t}^\star}{||\nabla f_{i_t}(\xx_t)||^2},\gamma_b\bigr\}$ where $\gamma_b>0$. Both algorithms require the input of the exact $f_{i_t}^\star$ which is often unavailable when the batch size $B>1$ or when the interpolation condition does not hold. \citet{decsps} removes the requirement for $f_{i_t}^\star$ and propose to set $\eta_t$ as: $\eta_t=\frac{1}{\sqrt{t+1}}\min\Big\{\frac{f_{i_t}(\xx_t)-\ell_{i_t}^\star}{||\nabla f_{i_t}(\xx_t)||^2},\sqrt{t}\eta_{t-1}\Big\}$ for $t\ge1$ (DecSPS), where $\eta_0>0$ is a constant and $\ell_{i_t}^\star$ is an input lower bound such that $\ell_{i_t}^\star\le f_{i_t}^\star$. In contrast to the exact optimal function value, a lower bound $\ell_{i_t}^\star$ is often available in practice, in particular for machine learning problems when the individual loss functions are non-negative. We henceforth denote the estimation error by:
\begin{equation}
{\rm err}^2_{f,B}:=\E_{i_t}[f_{i_t}^\star-\ell_{i_t}^\star]\,.
\end{equation}

For convex functions, SPS achieves a fast convergence up to a neighborhood of size $\Omega(\sigma_{f,B}^2)$ and its variant SPS$_{\max}$ converges up to  $\Omega(\sigma_{f,B}^2\gamma_b /\alpha)$ where $\alpha=\min\{\frac{1}{L},\gamma_b\}$. Note that the size of the neighborhood cannot be further reduced by choosing an appropriate $\gamma_b$. In contrast, DecSPS converges at the rate of $\cO(1/\sqrt{T})$ which matches the standard result for SGD with decreasing stepsize. 
However, the strictly decreasing $\cO(1/\sqrt{t})$ stepsize schedule hurts its performance in interpolated settings. For example, DecSPS has a much slower $\cO(1/\sqrt{T})$ convergence rate compared with the fast $\cO(\exp(-T\mu/L))$ rate of SPS when optimizing strongly-convex objectives. Therefore, both algorithms do not have the \emph{robust} convergence property (achieving fast convergence guarantees in both interpolated and non-interpolated regimes) and we aim to fill this gap.

\section{Adaptive SGD with polyak stepsize and line-search}
\label{sec:adasps}
In this section, we introduce and analyze two adaptive algorithms to solve  problem \eqref{eq:problem}.

\subsection{Proposed methods}
\textbf{AdaSPS.} 
Our first stepsize is defined as the following: 
\begin{align}
\eta_t=\min\left\{\frac{f_{i_t}(\xx_t)-\ell_{i_t}^\star}{c_p||\nabla f_{i_t}(\xx_t)||^2}\frac{1}{\sqrt{\sum_{s=0}^t f_{i_s}(\xx_s)-\ell_{i_s}^\star}},\eta_{t-1}\right\}\,, & & \text{with }\eta_{-1}=+\infty\,, \label{AdaSPS}
\tag{AdaSPS}
\end{align}
where $\ell_{i_t}^\star$ is an input parameter that must satisfy $\ell_{i_t}^\star\le f_{i_t}^\star$ and $c_p>0$ is an input constant to adjust the magnitude of the stepsize (we discuss suggested choices in Section \ref{sec:ex}).

AdaSPS can be seen as an extension of DecSPS. However, unlike the strict $\cO(1/\sqrt{t})$ decreasing rule applied in DecSPS, AdaSPS accumulates the function value difference during the optimization process which enables it to dynamically adapt to the underlying unknown interpolation settings. 

\textbf{AdaSLS.}
We provide another stepsize that can be applied even when a lower bound estimation is  unavailable. The method is based on line-search and thus is completely parameter-free, but requires additional function value evaluations in each iteration:
\begin{align}
 \eta_t=\min\left\{\frac{\gamma_t}{c_l\sqrt{\sum_{s=0}^t \gamma_{s}||\nabla f_{i_s}(\xx_s)||^2}},\eta_{t-1}\right\}\,, & & \text{with }\eta_{-1}=+\infty\,,
\label{AdaSLS}
\tag{AdaSLS}
\end{align}
where $c_l>0$ is an input constant, and the scale $\gamma_t$ is obtained via stardard Armijo line-search (see Algorithm~\ref{Algorithm:Armijo} for further implementation details  in the Appendix~\ref{sec:line-search}) such that the following conditions are satisfied:
\begin{equation}
f_{i_t}(\xx_t-\gamma_t\nabla f_{i_t}(
\xx_t))\le f_{i_t}(\xx_t)-\rho\gamma_t||\nabla f_{i_t}(\xx_t)||^2\quad\text{and}\quad\gamma_t\le\gamma_{\max},\;0<\rho<1\;,
\label{eq:Armijo}
\end{equation}
for line search parameters $\gamma_{\rm max}$ and $\rho$. 

\textbf{Discussion.} 
Our adaptation mechanism in AdaSPS/AdaSLS is reminiscent of AdaGrad type methods, in particular to AdaGrad-Norm, the scalar version of AdaGrad, that aggregates the gradient norm in the denominator and takes the form $\eta_t=\frac{c_g}{\sqrt{\sum_{s=0}^t||\nabla f_{i_s}(\xx_s)||^2+b_0^2}}$ where $c_g>0$ and $b_0^2\ge0$.

The primary distinction between AdaSPS and AdaSLS compared to AdaGrad-Norm is the inclusion of an additional component that captures the curvature information at each step, and not using squared gradient norms in AdaSPS.  In contrast to the strict decreasing behavior of AdaGrad-Norm, AdaSPS and AdaSLS can automatically mimic a constant stepsize when navigating a flatter region. 

\citet{adagrad+ls} suggest using line-search to set the stepsize for AdaGrad-Norm which takes the form $\eta_t=\frac{\gamma_t}{\sqrt{\sum_{s=0}^t||\nabla f_{i_s}(\xx_s)||^2}}$ where $\gamma_{t}\le\gamma_{t-1}$ is required for solving non-interpolated convex problems. While this stepsize is similar to AdaSLS, the scaling of the denominator gives a suboptimal convergence rate as we demonstrate in the following section.

\subsection{Convergence rates}
In this section, we present the convergence results for AdaSPS and AdaSLS. We list the helpful lemmas in Appendix \ref{sec:pf_lm}. The proofs can be found in Appendix \ref{sec:pf_main}.

\textbf{General convex.} 
We denote $\cX$ to be a convex compact set with diameter $D$ such that there exists a solution $\xx^\star\in \cX$ and $\sup_{\xx,\yy\in \cX}||\xx-\yy||^2\le D^2$. We let $\Pi_\cX$ denote the Euclidean projection onto $\cX$. For general convex stochastic optimization, it seems inevitable that adaptive methods require the bounded iterates assumption or an additional projection step to prove convergence due to the lack of knowledge of problem-dependent parameters \cite{unbound_adp,AdaGrad,wald,decsps,ada-svrg}. Here, we employ the latter solution by running projected stochastic gradient descent (PSGD):
\begin{equation}
\xx_{t+1} = \Pi_\cX(\xx_t - \eta_t \nabla f_{i_t}(\xx_t)).
\label{eq:PSGD}
\end{equation}
\begin{theorem}[General convex]
Assume that $f$ is convex, each $f_i$ is $L$-smooth and $\cX$ is a convex compact feasible set with diameter $D$,  PSGD with AdaSPS or AdaSLS converges as:
\begin{equation}
\begin{split}
    &(\text{AdaSPS}):\;\E[f(\Bar{\xx}_T)-f^\star]\le\frac{\tau_p^2}{T}+\frac{\tau_p\sqrt{\sigma_{f,B}^2+{\rm err}_{f,B}^2}}{\sqrt{T}}\;, \\ 
    &(\text{AdaSLS}):\;\E[f(\Bar{\xx}_T)-f^\star]\le\frac{\tau_l^2}{T}+\frac{\tau_l\sigma_{f,B}}{\sqrt{T}}\;,
    \label{eq:thm:convex}
\end{split}
\end{equation}
where $\Bar{\xx}_T=\frac{1}{T}\sum_{t=0}^{T-1}\xx_t$, $\tau_p=(2c_pLD^2+\frac{1}{c_p})$ and $\tau_l=\max\Bigl\{\frac{L}{(1-\rho)\sqrt{\rho}},\frac{1}{\gamma_{\max}\sqrt{\rho}}\Bigr\}c_lD^2+\frac{1}{c_l\sqrt{\rho}}$.
\label{thm:convex}
\end{theorem}

As a consequence of Theorem \ref{thm:convex}, if ${\rm err}_{f,B}^2=\sigma_{f,B}^2=0$, then PSGD with AdaSPS or AdaSLS converges as $\cO(\frac{1}{T})$. Suppose $\gamma_{\max}$ is sufficiently large, then picking $c_p^\star=\frac{1}{\sqrt{2LD^2}}$ and $c_l^\star=\frac{\sqrt{1-\rho}}{\sqrt{LD^2}}$ gives a $\cO(\frac{LD^2}{T})$ rate under the interpolation condition, which is slightly worse than $\frac{L||\xx_0-\xx^\star||^2}{T}$ obtained by SPS and SLS but is better than $\cO(\frac{LD^2}{\sqrt{T}})$ obtained by DecSPS. If otherwise $\sigma_{f,B}^2>0$, then AdaSPS, AdaSLS, and DecSPS converge as $\cO(1/\sqrt{T})$ which matches the rate of Vanilla SGD with decreasing stepsize. Finally, AdaGrad-Norm gives a similar rate in both cases while AdaGrad-Norm with line-search \cite{adagrad+ls} shows a suboptimal rate of $\cO(\frac{L^3D^4}{T}+\frac{D^2L^{3/2}\sigma}{\sqrt{T}})$. It is worth noting that SPS, DecSPS and SLS require an additional assumption on individual convexity.  

\begin{theorem}[Individual convex+interpolation]

   Assume that $f$ is convex, each $f_i$ is convex and $L$-smooth, and that ${\rm err}_{f,B}^2=\sigma_{f,B}^2=0$, by setting $c_p=\frac{c_p^{\text{scale}}}{\sqrt{f_{i_0}(\xx_0)-f_{i_0}^\star}}$ and $c_l=\frac{c_l^{\text{scale}}}{\rho\sqrt{\gamma_0||\nabla f_{i_0}(\xx_0)||^2}}$ where the constants $c_p^{\text{scale}}\ge 1$ and $c_p^{\text{scale}}\ge 1$, then for any $T\ge1$, SGD (no projection) with AdaSPS or AdaSLS converges as: 
\begin{equation}
   \text{(AdaSPS)} \quad \E[f(\Bar{\xx}_T)-f^\star]\le \Biggl(4L(c_p^{\text{scale}})^2\E_{i_0}\Bigl[\frac{||\xx_0-\xx^\star||^2}{f_{i_0}(\xx_0)-f^\star}\Bigr]\Biggr)\frac{L||\xx_0-\xx^\star||^2}{T}\;.
\end{equation}
and 
\begin{equation}
   \text{(AdaSLS)} \quad \E[f(\Bar{\xx}_T)-f^\star]\le \Biggl(\frac{(c_l^{\text{scale}})^2}{\rho^3L\min^2\{\frac{1-\rho}{L},\gamma_{\max}\}}\E_{i_0}\Bigl[\frac{||\xx_0-\xx^\star||^2}{\gamma_0||\nabla f_{i_0}(\xx_0)||^2}\Bigr]\Biggr)\frac{L||\xx_0-\xx^\star||^2}{T}\;.
\end{equation}
where $\Bar{\xx}_T=\frac{1}{T}\sum_{t=1}^{T}\xx_t$,
\label{theorem:convex+interp}
\end{theorem}

The result above implies that the bounded iterates assumption is not needed for individual convexity and interpolated settings by picking $c_p$ and $c_l$ to satisfy certain conditions that do not depend on unknown parameters. To our  knowledge, no such result exists for stepsizes from the AdaGrad family. It is also worth noting that the $\min$ operator defined in AdaSPS or AdaSLS is not needed in the proof.

\begin{remark}
    We note that for non-interpolated problems, AdaSPS only requires the knowledge of $\ell_{i_t}^\star$ while the exact $f_{i_t}^\star$ is needed under the interpolation condition. We argue that in many standard machine learning problems, simply picking zero will suffice. For instance, $f_{i_t}^\star=0$ for over-parameterized logistic regression and  after adding a regularizer, $\ell_{i_t}^\star=0$.
\end{remark}

\textbf{Strongly convex.} We now present two algorithmic behaviors of AdaSPS and AdaSLS for strongly convex functions. In particular,
We show that 1) the projection step can be removed as shown in DecSPS, and 2) if the interpolation condition holds, the $\min$ operator is not needed and the asymptotic linear convergence rate is preserved. The full statement of Lemma \ref{lemma:bounded iterates} can be found in Appendix \ref{appendix:sec:lemma3}.
\begin{lemma}[Bounded iterates]
    Let each $f_i$ be $\mu$-strongly convex and $L$-smooth. For any $t=0, \dots, T$, the iterates of SGD with AdaSPS or AdaSLS satisfy: $||\xx_t-\xx^\star||^2\le D_{\max}$, for a constant $D_{\max}$ specified in the appendix in Equation~\eqref{eq:dmax}.
    \label{lemma:bounded iterates}
\end{lemma}

\begin{corollary}[Individual strongly convex]
    Assume each $f_i$ is $\mu$-strongly convex and $L$-smooth, Theorem~\ref{thm:convex} holds with PSGD and $D$ replaced by SGD and $D_{\max}$ defined in Lemma \ref{lemma:bounded iterates}.
\end{corollary}
Although it has not been formally demonstrated that AdaGrad/AdaGrad-Norm can relax the assumption on bounded iterates for strongly convex functions, we believe that with a similar proof technique, this property still holds for AdaGrad/AdaGrad-Norm.

We next show that AdaSPS and AdaSLS achieve linear convergence under the interpolation condition. 
\begin{theorem}[Strongly convex + individual convex + interpolation]
    Consider SGD with AdaSPS \eqref{AdaSPS} or AdaSLS \eqref{AdaSLS} stepsize. Suppose that each $f_i$ is convex and $L$-smooth and that $\sigma_{f,B}^2={\rm err}_{f,B}^2=0$. If we let $c_p=\frac{c_p^{\text{scale}}}{\sqrt{f_{i_0}(\xx_0)-f_{i_0}^\star}}$ and $c_l=\frac{c_l^{\text{scale}}}{\rho\sqrt{\gamma_0||\nabla f_{i_0}(\xx_0)||^2}}$ where the constants $c_p^{\text{scale}}\ge 1$ and $c_p^{\text{scale}}\ge 1$, then for any $T\ge 1$, it holds that:
    \begin{equation}
    \text{(AdaSPS)} \quad \E[||\xx_{T+1}-\xx^\star||^2]\le\E_{i_0}\Bigg[\Bigl(1-\frac{(f_{i_0}(\xx_0)-f^\star)\mu}{(2c_p^{\text{scale}}L||\xx_0-\xx^\star||)^2}\Bigr)^T\Biggr]||\xx_{0}-\xx^\star||^2\;,
    \end{equation}
    and
    \begin{equation}
    \text{(AdaSLS)} \quad \E[||\xx_{T+1}-\xx^\star||^2]\le\E_{i_0}\Biggl[\Bigl(1-\frac{\mu\rho^3\min^2\{\frac{1-\rho}{L},\gamma_{\max}\}\gamma_0||\nabla f_{i_0}(\xx_0)||^2}{(c_l^{\text{scale}}||\xx_0-\xx^\star||)^2}\Bigr)^T\Biggr]||\xx_0-\xx^\star||^2\;.
    \end{equation}
\label{thm:stconvex}
\end{theorem}

We now compare the above results with the other stepsizes. Under the same settings, DecSPS has a slower $\cO(1/\sqrt{T})$ convergence rate due to the usage of $\cO(1/\sqrt{t})$ decay stepsize. While AdaGrad-Norm does have a linear acceleration phase when the accumulator grows large, to avoid an $\cO(1/\epsilon)$ slow down, the parameters of AdaGrad-Norm have to satisfy $c_g<b_0/L$, which requires the knowledge of Lipschitz constant \cite{linear-AdaGrad}. Instead, the conditions on $c_p$ and $c_l$ for AdaSPS and AdaSLS only depend on the function value and gradient norm at $\xx_0$ which can be computed at the first iteration. SPS, SLS and Vannilia-SGD with constant stepsize achieve fast linear convergence with a better conditioning number $\exp(-\frac{\mu}{L}T)$. It is worth noting that Vannila-SGD can further remove the individual convexity assumption. 

\textbf{Discussion.} In non-interpolation regimes, AdaSPS and AdaSLS only ensure a slower $\cO(1/\sqrt{t})$ convergence rate compared with $\cO(1/t)$ rate achieved by vanilla SGD with $\cO(1/t)$ decay stepsize when optimizing strongly-convex functions \cite{large_opt_book}. To our knowledge, no parameter-free adaptive stepsize exists that achieves such a fast rate under the same assumptions. While one may consider removing the square root operator defined in the denominator, the knowledge of $\mu$ constant is still necessary for setting the stepsize \cite{log-AdaGrad} and the resulting algorithm is no longer a unified adaptive algorithm. Therefore, discovering an accumulation rule that can adapt to both convex and strongly-convex functions would be a significant further contribution.

\section{AdaSPS and AdaSLS with variance-reduction}
\label{sec:adasvrps}

The classical VR algorithms such as SVRG \cite{VR_zhang} construct a sequence of variance-reduced stochastic gradients: $g_t:=\nabla f_{i_t}(\xx_t)+\nabla f(\ww_t)-\nabla f_{i_t}(\ww_t)$ where $\ww_t$ denotes some past iterate. This  sequence has diminishing variance, i.e.\ $\E[||g_t-\nabla f(\xx_t)||^2]\to 0$ as $\xx_t\to\xx^\star$. While the algorithms from the AdaGrad family  have been successfully applied to setting the stepsize for $g_t$ \cite{ada-svrg}, intuitive SPS/SLS-like stepsizes can instead prevent convergence (in Appendix \ref{appendix:sec:counter_ex} we provide counter-examples of SPS and its variants, \cite{ada-svrg} provide a counter-example for line-search methods). 

\subsection{Algorithm design: achieving variance-reduction without interpolation}
It is known that adaptive methods such as SPS or SLS converge linearly on problems where the interpolation condition holds, i.e.\ $f(\xx)$ with $\sigma_{f,B}=0$.

For problems that do not satisfy the interpolation condition, our approach is to transition the problem to an equivalent one that satisfies the interpolation condition.
One such transformation is to shift each individual function by the gradient of $f_i(\xx)$ at $\xx^\star$, i.e.\ $F_i(\xx)=f_i(\xx)-\xx^T\nabla f_i(\xx^\star)$. In this case $f(\xx)$ can be written as $f(\xx)=\frac{1}{n}\sum_{i=1}^n F_i(\xx)$ due to the fact that $\frac{1}{n}\sum_{i=1}^n\nabla f_i(\xx^\star)=0$. Note that $\nabla F_i(\xx^\star)=\nabla f_i(\xx^\star)-\nabla f_i(\xx^\star)=0$ which implies that each $F_i(\xx)$ shares the same minimizer and thus the interpolation condition is  satisfied ($\sigma_{f,1}^2=0$). 
However, $\nabla f_i(\xx^\star)$ is usually not available at hand. This motivates us to design the following algorithm.

\subsection{Algorithms and convergence}
Inspired by this observation, we attempt to reduce the variance of the functions $\sigma_{f,B}^2$ by constructing a sequence of random functions $\{F_{i_t}(\xx)\}$ such that $\sigma_{\frac{1}{n}\sum_{i=1}^nF_{i_t}(\xx),B}^2\to0$ as $\xx_t\to\xx^\star$. A similar idea can also be found in the works on federated learning with variance-reduction \cite{feddyn,scaffold,proxskip}. However, directly applying SPS or SLS to $\{F_{i_t}(\xx)\}$ still requires the knowledge of the Lipschitz constant to guarantee convergence. We address this problem by using our proposed AdaSPS and AdaSLS. The whole procedure of the final algorithm is summarized in Algorithm~\ref{AdaVR}.

\begin{algorithm}[tb]
\begin{algorithmic}[1]
\Require $\xx_0\in\R^d$, $\mu_F>0$, $c_p>0$ or $c_l>0$
\State set $\ww_{0}=\xx_0$, $\eta_{-1}=+\infty$
\For{$t=0$ to $T-1$}
\State uniformly sample $i_t\subseteq[n]$\smallskip
\State set $F_{i_t}(\xx)=f_{i_t}(\xx)+\xx^T(\nabla f(\ww_t)-\nabla f_{i_t}(\ww_t))+\frac{\mu_F}{2}||\xx-\xx_t||^2$\smallskip
\State $\eta_t=\min\Bigl\{\frac{F_{i_t}(\xx_t)-F_{i_t}^\star}{c_p||\nabla F_{i_t}(\xx_t)||^2}\frac{1}{\sqrt{\sum_{s=0}^t F_{i_s}(\xx_s)-F_{i_s}^\star}},\eta_{t-1}\Bigr\}$ (AdaSVRPS)\smallskip
\State $\eta_t=\min\Bigl\{\gamma_t\frac{1}{c_l\sqrt{\sum_{s=0}^t \gamma_{s}||\nabla F_{i_s}(\xx_s)||^2}},\eta_{t-1}\Bigr\}$ (AdaSVRLS)\footnotemark{}\smallskip
\State $\xx_{t+1}=\Pi_{\cX}\bigl(\xx_t-\eta_t\nabla F_{i_t}(\xx_t)\bigr)$\smallskip
\State $\ww_{t+1}=\begin{cases}
    \ww_{t}&\; \text{with probability}\;1-p_{t+1} \\
    \xx_{t}&\; \text{with probability}\;p_{t+1}
    \end{cases}$
\EndFor\\
\Return $\Bar{\xx}_T=\frac{1}{T}\sum_{t=0}^{T-1}\xx_t$
\caption{(Loopless) AdaSVRPS and AdaSVRLS}
\label{AdaVR}
\end{algorithmic}
\end{algorithm}
\footnotetext{where $\gamma_t$ is obtained via the standard Armijo line-search (Algorithm~\ref{Algorithm:Armijo}) which satisfies: $F_{i_t}(\xx_t-\gamma_t\nabla F_{i_t}(\xx_t))\le F_{i_t}(\xx_t)-\rho\gamma_t||\nabla F_{i_t}(\xx_t)||^2$ and $\gamma_t\le\gamma_{\max}$.}

At each iteration of Algorithm~\ref{AdaVR}, we construct a proxy function by adding two quantities to the minibatch function $f_{i_t}(\xx)$, where $\frac{\mu_F}{2}||\xx-\xx_t||^2$ is a proximal term that helps improve the inherent stochasticity due to the partial information obtained from $f_{i_t}(\xx)$. Note that adding a proximal term has been successfully applied to federated learning and it effectively addresses the systems heterogeneity issues \cite{fedprox}. The additional inner product quantity is used to draw closer the minimizers of $f_{i_t}(\xx)$ and  $f(\xx)$. Following \cite{loopless_svrg,page}, the full gradient is computed with a coin flip probability. Note that Algorithm~\ref{AdaVR} still works with $\eta_t$ replaced with SVRG and AdaSVRG stepsize since $\E_{i_t}[\nabla F_{i_t}(\xx_t)]=\nabla f(\xx_t)$, and thus this framework can be seen as a generalization of the standard VR methods. This may allow us to combine more optimization algorithms with VR techniques. 

\begin{lemma}
Assume each $f_i$ is convex and $L$-smooth, for any $t\ge0$, the iterates generated by Algorithm \ref{AdaVR} satisfy:
\begin{equation}
    \E_{i_t}[F_{i_t}(\xx_t)-F_{i_t}^\star]\le f(\xx_t)-f^\star+\frac{1}{2\mu_F}\E_{i_t}\bigl[||\nabla f_{i_t}(\ww_t)-\nabla f_{i_t}(\xx^\star)||^2\bigr] \,.
    \label{eq:lemma:vr}
\end{equation}
\label{lemma:vr}
\end{lemma}
\vspace{-\baselineskip}
Lemma \ref{lemma:vr} essentially provides us with the guarantee that as $\ww_t,\xx_t\to\xx^\star$, $ \E_{i_t}[F_{i_t}(\xx_t)-F_{i_t}^\star]\to0$, which implies diminishing variance. We are ready to establish the convergence rate for Algorithm \ref{AdaVR}.

\begin{theorem}
     Assume each $f_i$ is convex and $L$ smooth and $\cX$ is a convex compact feasible set with diameter $D$. Let $p_{t}=\frac{1}{at+1}$ with $0\le a <1$. Algorithm \ref{AdaVR} converges as:
     \begin{align}
    &\text{(AdaSVRPS)}\; & &\E[f(\Bar{\xx}_T)-f^\star]
    \le\frac{1+\frac{2L}{(1-a)\mu_F}}{T}\Bigl(2c_p(L+\mu_F)D^2+\frac{1}{c_p}\Bigr)^2 \;,
    \end{align}
    \begin{equation}
    \text{(AdaSVRLS)}\; \E[f(\Bar{\xx}_T)-f^\star]\le\frac{1+\frac{2L}{(1-a)\mu_F}}{T}\Bigl(\max\Bigl\{\frac{L+\mu_F}{(1-\rho)\sqrt{\rho}},\frac{1}{\gamma_{\max}\sqrt{\rho}}\Bigr\}c_lD^2+\frac{1}{c_l\sqrt{\rho}}\Bigr)^2\;,
\end{equation}
where $\Bar{\xx}_T=\frac{1}{T}\sum_{t=0}^{T-1}\xx_t$. 
\label{thm:vr}
\end{theorem}
Suppose $\gamma_{\max}$ is sufficiently large, then picking $\mu_F^\star=\cO(L)$, $c_p^\star=\cO(\frac{1}{\sqrt{LD^2}})$ and $c_l^\star=\cO(\frac{\sqrt{1-\rho}}{\sqrt{LD^2}})$ yields an $\cO(\frac{LD^2}{T})$ rate which matches the $\cO(\frac{L||\xx_0-\xx^\star||^2}{T})$ rate of full-batch gradient descent except for a larger term $D^2$ due to the lack of knowledge of the Lipschitz constant. 

\begin{corollary}
    Under the setting of Theorem \ref{thm:vr}, given an arbitrary accuracy $\epsilon$,  the total number of gradient evaluations required to have $\E[f(\Bar{\xx}_T)-f^\star]
    \le\epsilon$ in expectation is $\cO(\log(1/\epsilon)n+1/\epsilon)$.
\end{corollary}

\begin{proof}
Algorithm \ref{AdaVR} calls the stochastic gradient oracle in expectation $\cO(1+p_tn)$ times at iteration $t$. Therefore, the total number of gradient evaluations is upper bounded by $\cO(\sum_{t=0}^{T-1}p_tn+T)$. By our choice of $p_t$, it holds that $\sum_{t=0}^{T-1}p_t\le\frac{1}{a}\sum_{t=0}^{T-1}\frac{1}{t+2}\le\frac{1}{a}(\log(T)+1-1)=\frac{1}{a}\log(T)$. Due to the sublinear convergence rate of Algorithm \ref{AdaVR}, we conclude that the total number of stochastic gradient calls is $\cO(\log(1/\epsilon)n+1/\epsilon)$. 
\end{proof}

The proved efficiency of stochastic gradient calls matches the fast rates of SARAH \cite{sarah}/SVRG and AdaSVRG~\cite{ada-svrg}. Note that SVRG and SARAH require knowledge of the problem-dependent parameters. AdaSVRG needs to predefine the target accuracy $\epsilon$, to design the number of stages and the inner-outer-loop size. Furthermore, its convergence result cannot be extended to the case with arbitrary full gradient update frequency. Indeed, AdaSVRG fails to converge with inner-loop size being one and fixing $g_t=\nabla f(\xx_t)$. One drawback of Algorithm~\ref{AdaVR} is the additional assumption on individual convexity. Since SPS/SLS has to assume the same condition for the proof in the interpolation settings, we believe this assumption is unfortunately also necessary for variance-reduction.

\textbf{Discussion.} The classical SVRG with Armijo line-search (presented as Algorithm 6 in \cite{ada-svrg}) employs the same gradient estimator as SVRG but chooses its stepsize based on the returning value of line-search on the individual function $f_{i}$. Similarly, SVRG with classical Polyak stepsize uses the individual curvature information of $f_i$ to set the stepsize for the global variance-reduced gradient. Due to the misleading curvature information provided by the biased function $f_i$, both methods have convergence issues. In constrast, Algorithm~\ref{AdaVR} reduces the bias by adding a correction term $\xx^T(\nabla f(\ww_t)-\nabla f_{i_t}(\ww_t))$ with global information to $f_i$ and then applying line-search or Polyak-stepsize on the variance-reduced functions $F_{i_t}$. This main difference essentially guarantees the convergence.

\section{Numerical evaluation}
\label{sec:ex}
In this section, we illustrate the main properties of our proposed methods in numerical experiments. Due to space limitations, we only showcase AdaSPS/AdaSVRPS results in the main text. The study of AdaSLS/AdaSVRLS is deferred to Appendix \ref{appendix:ex} where also a detailed description of the experimental setup can be found. We define the theoretically justified hyperparameters $c_p^{\text{scale}}:=c_p\sqrt{f_{i_0}(\xx_0)-\ell_{i_0}^\star}\ge1$ for AdaSPS and $c_p^{\text{scale}}:=c_p\sqrt{F_{i_0}(\xx_0)-F_{i_0}^\star}\ge 1$ for AdaSVRPS. 
The discussion about replacing $F_{i_0}^\star$ with $\ell_{i_0}^\star$ can be found in Appendix \ref{appendix:ex}.

\textbf{Synthetic data.} We illustrate the robustness property on a class of synthetic problems.
We consider the minimization of a quadratic of the form: $f(\xx)=\frac{1}{n}\sum_{i=1}^nf_i(\xx)$ where $f_i(\xx)=\frac{1}{2}(\xx-\bb_i)^TA_i(\xx-\bb_i)$, $\bb_i\in\R^d$ and $A_i\in\R^{d\times d}$ is a diagonal matrix. We use $n=50$, $d=1000$. 
We can control the convexity of the problem by choosing different matrices $A_i$, and control interpolation by either setting all  $\{\bb_i\}$  to be identical or different. We generate a strongly convex instance where the eigenvalues of  $\nabla^2 f(\xx)$ are between $1$ and $10$, and a general convex instance by setting some of the eigenvalues to small values close to zero (while ensuring that each $\nabla^2 f_i(\xx)$ is positive semi-definite). The exact procedure to generate these problems is described in Appendix~\ref{appendix:ex}. 

 For all methods, we use a batch size $B=1$. 
 For AdaSPS/AdaSVRPS we fix $c_p^{\text{scale}}=1$, and for AdaSVRPS we further use $\mu_F=10$ and $p_t=\frac{1}{0.1t+1}$.  We compare against DecSPS~\cite{decsps}, SPS~\cite{sps} and SVRG~\cite{VR_zhang} and tune the stepsize for SVRG by picking the best one from $\{10^i\}_{i=-4,..,3}$.
 In Figure \ref{fig:syn}, we observe that SPS does not converge in the non-interpolated settings and DecSPS suffers from a slow $\cO(1/\sqrt{T})$ convergence on the two interpolated problems. AdaSPS shows the desired convergence rate across all cases which matches the theory while AdaSVRPS further improves the performance in the non-interpolated regime. In Figure~\ref{fig:syn} we also note that AdaSPS can automatically adjust the stepsize according to the underlying interpolation condition while DecSPS always follows an $\cO(1/\sqrt{t})$ decreasing stepsize schedule.

\begin{figure*}[tb!]
    \centering
    \includegraphics[width=1\textwidth]{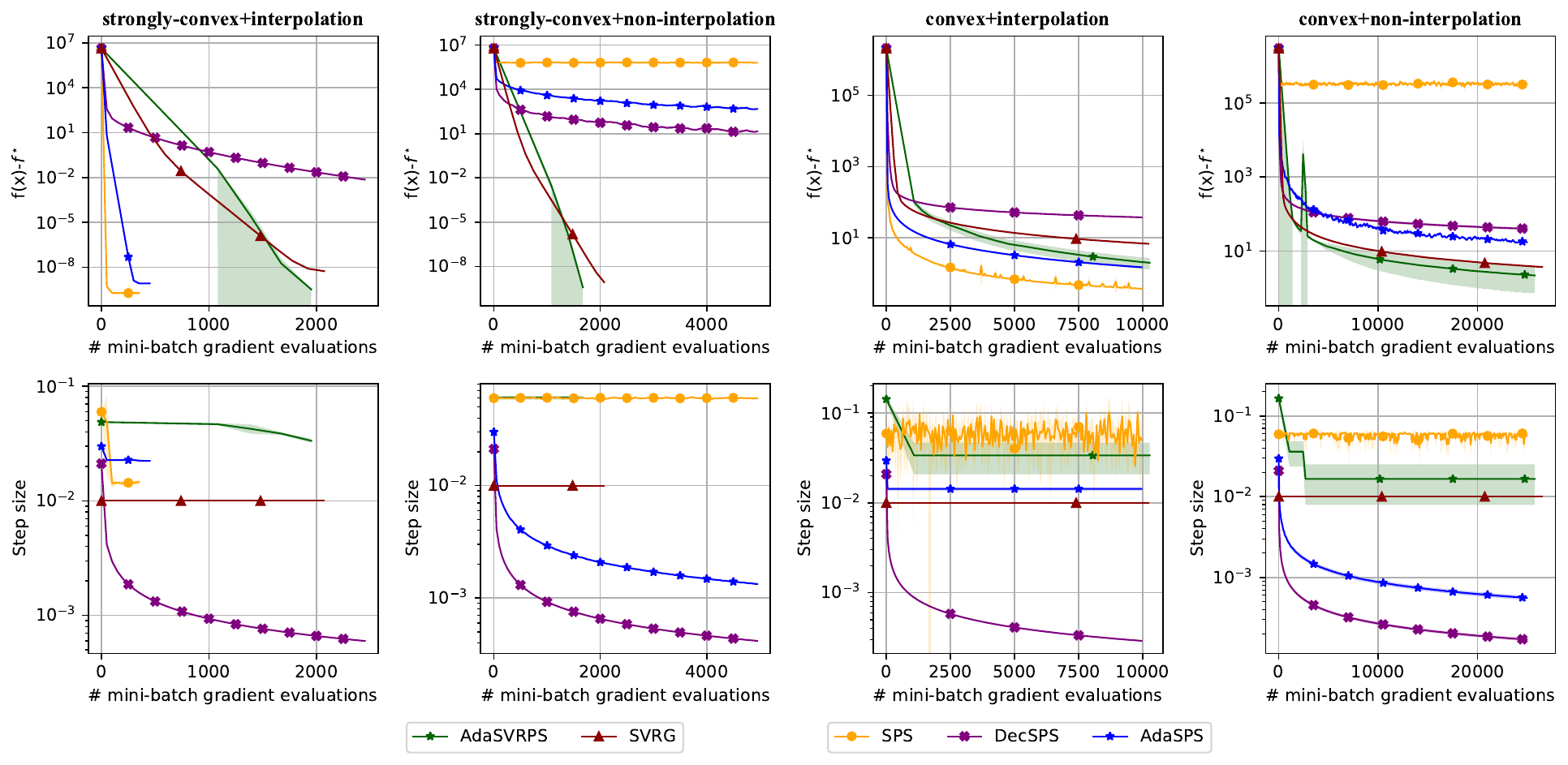}  
    \vspace*{-6mm}
    \caption{Illustration of the robust convergence of AdaSPS on synthetic data with quadratic loss. SPS has superior performance on the two interpolated problems but cannot converge when the interpolation condition does not hold. DecSPS suffers from a slow convergence on both interpolated problems. AdaSVRPS and SVRG show remarkable performances when solving non-interpolated problems. (Repeated 3 times. The solid lines and the shaded area represent the mean and the standard deviation.)}
    \label{fig:syn}
\end{figure*}

\textbf{Binary classification on LIBSVM datasets.}
We experiment with binary classification on four diverse datasets from \cite{libsvm}. We consider the standard regularized logistic loss: $f(\xx)=\frac{1}{n}\sum_{i=1}^n\log(1+\exp(-y_i\cdot \aa_i^T\xx))+\frac{1}{2n}||\xx||^2$ where $(\aa_i,y_i)\in\R^{d+1}$ are features and labels. We benchmark against popular optimization algorithms including Adam~\cite{adam}, SPS~\cite{sps}, DecSPS~\cite{decsps}, AdaGrad-Norm~\cite{AdaGrad}, SVRG~\cite{VR_zhang} and AdaSVRG~\cite{ada-svrg}. We report the best $c_{p}^{\text{scale}}\in\{0.5,1,2\}$, $\mu_F\in\{10^{-4},10^2\}$ and the best learning rate from $\{10^i\}_{i=-4,..,3}$ for SGD, Adam and AdaGrad-Norm. We compute the full gradient at the beginning of each epoch for SVRG and AdaSVRG. We report $p_t=\frac{B}{n}$ for AdaSVRPS. Comparisons with decreasing probability framework can be found in Appendix \ref{appendix:ex} with different choices of batch size. We observe that Adam and SPS have remarkable performances on duke with $n=48$ and $d=7129$ which clearly satisfies interpolation. AdaSPS consistently performs well on the other three larger datasets. Finally, three VR methods give comparable performances when interpolation is not satisfied.

\textbf{Discussion.} AdaSPS and AdaSVRPS consistently demonstrate robust convergence across all tasks, and achieve performance on par with, if not better than, the best-tuned algorithms. Consequently, these optimizers can be considered reliable for practical use.

\begin{figure*}[tb!]
    \centering
    \includegraphics[width=1\textwidth]{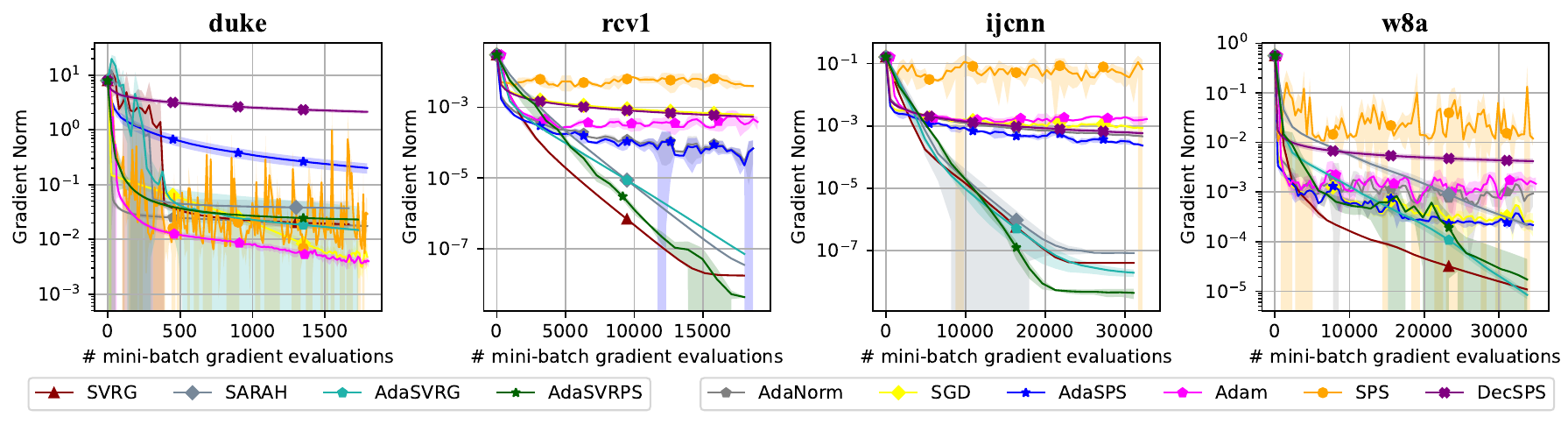}  
    \vspace*{-6mm}
    \caption{Comparison of AdaSPS/AdaSVRPS against seven other popular optimizers on four LIBSVM datasets, with batch size $B=1$ for duke, $B=64$ for rcv1, $B=64$ for ijcnn and $B=128$ for w8a. AdaSPS and AdaSVRPS have competitive performance on rcv1, ijcnn and w8a while SPS and Adam converge fast on duke. (Repeated 3 times. The solid lines and the shaded area represent the mean and the standard deviation.)}
    \label{fig:libsvm}
\end{figure*}

\section{Conclusion and future work}

We proposed new variants of SPS and SLS algorithms and demonstrated their robust and fast convergence in both interpolated and non-interpolated settings. We further accelerate both algorithms for convex optimization with a novel variance reduction technique. Interesting future directions may include: accelerating AdaSPS and AdaSLS with momentum, developing effective robust adaptive methods for training deep neural networks, designing an adaptive algorithm that gives a faster rate $\cO(1/T)$ under strong convexity, extensions to distributed and decentralized settings. 

\newpage
\bibliographystyle{plainnat}
{\small
\bibliography{reference}

\begin{thebibliography}{48}
\providecommand{\natexlab}[1]{#1}
\providecommand{\url}[1]{\texttt{#1}}
\expandafter\ifx\csname urlstyle\endcsname\relax
  \providecommand{\doi}[1]{doi: #1}\else
  \providecommand{\doi}{doi: \begingroup \urlstyle{rm}\Url}\fi

\bibitem[Acar et~al.(2021)Acar, Zhao, Navarro, Mattina, Whatmough, and
  Saligrama]{feddyn}
Durmus Alp~Emre Acar, Yue Zhao, Ramon~Matas Navarro, Matthew Mattina, Paul~N
  Whatmough, and Venkatesh Saligrama.
\newblock Federated learning based on dynamic regularization.
\newblock \emph{arXiv preprint arXiv:2111.04263}, 2021.

\bibitem[Armijo(1966)]{armijo-line-search}
Larry Armijo.
\newblock Minimization of functions having lipschitz continuous first partial
  derivatives.
\newblock \emph{Pacific Journal of Mathematics}, 16\penalty0 (1):\penalty0
  1--3, 1 1966.

\bibitem[Berrada et~al.(2020)Berrada, Zisserman, and Kumar]{alig}
Leonard Berrada, Andrew Zisserman, and M.~Pawan Kumar.
\newblock Training neural networks for and by interpolation.
\newblock In \emph{Proceedings of the 37th International Conference on Machine
  Learning}, ICML'20. JMLR.org, 2020.

\bibitem[Bottou et~al.(2018)Bottou, Curtis, and Nocedal]{large_opt_book}
L\'{e}on Bottou, Frank~E. Curtis, and Jorge Nocedal.
\newblock Optimization methods for large-scale machine learning.
\newblock \emph{SIAM Review}, 60\penalty0 (2):\penalty0 223--311, 2018.
\newblock \doi{10.1137/16M1080173}.
\newblock URL \url{https://doi.org/10.1137/16M1080173}.

\bibitem[Bubeck(2015)]{bubek_convex}
S\'{e}bastien Bubeck.
\newblock Convex optimization: Algorithms and complexity.
\newblock \emph{Found. Trends Mach. Learn.}, 8\penalty0 (3–4):\penalty0
  231–357, nov 2015.
\newblock ISSN 1935-8237.
\newblock \doi{10.1561/2200000050}.
\newblock URL \url{https://doi.org/10.1561/2200000050}.

\bibitem[Chang and Lin(2011)]{libsvm}
Chih-Chung Chang and Chih-Jen Lin.
\newblock Libsvm: A library for support vector machines.
\newblock \emph{ACM Trans. Intell. Syst. Technol.}, 2\penalty0 (3), may 2011.
\newblock ISSN 2157-6904.
\newblock \doi{10.1145/1961189.1961199}.
\newblock URL \url{https://doi.org/10.1145/1961189.1961199}.

\bibitem[Cutkosky and Boahen(2016)]{unbound_adp}
Ashok Cutkosky and Kwabena Boahen.
\newblock Online convex optimization with unconstrained domains and losses.
\newblock In \emph{Proceedings of the 30th International Conference on Neural
  Information Processing Systems}, NIPS'16, page 748–756, Red Hook, NY, USA,
  2016. Curran Associates Inc.
\newblock ISBN 9781510838819.

\bibitem[Defazio et~al.(2014)Defazio, Bach, and Lacoste-Julien]{saga}
Aaron Defazio, Francis Bach, and Simon Lacoste-Julien.
\newblock Saga: A fast incremental gradient method with support for
  non-strongly convex composite objectives.
\newblock In Z.~Ghahramani, M.~Welling, C.~Cortes, N.~Lawrence, and K.Q.
  Weinberger, editors, \emph{Advances in Neural Information Processing
  Systems}, volume~27. Curran Associates, Inc., 2014.
\newblock URL
  \url{https://proceedings.neurips.cc/paper_files/paper/2014/file/ede7e2b6d13a41ddf9f4bdef84fdc737-Paper.pdf}.

\bibitem[Dubois{-}Taine et~al.(2022)Dubois{-}Taine, Vaswani, Babanezhad,
  Schmidt, and Lacoste{-}Julien]{ada-svrg}
Benjamin Dubois{-}Taine, Sharan Vaswani, Reza Babanezhad, Mark Schmidt, and
  Simon Lacoste{-}Julien.
\newblock {SVRG} meets adagrad: painless variance reduction.
\newblock \emph{Mach. Learn.}, 111\penalty0 (12):\penalty0 4359--4409, 2022.
\newblock \doi{10.1007/s10994-022-06265-x}.
\newblock URL \url{https://doi.org/10.1007/s10994-022-06265-x}.

\bibitem[Duchi et~al.(2011)Duchi, Hazan, and Singer]{AdaGrad}
John Duchi, Elad Hazan, and Yoram Singer.
\newblock Adaptive subgradient methods for online learning and stochastic
  optimization.
\newblock \emph{J. Mach. Learn. Res.}, 12\penalty0 (null):\penalty0
  2121–2159, jul 2011.
\newblock ISSN 1532-4435.

\bibitem[Galli et~al.(2023)Galli, Rauhut, and Schmidt]{relax_ls}
Leonardo Galli, Holger Rauhut, and Mark Schmidt.
\newblock Don't be so monotone: Relaxing stochastic line search in
  over-parameterized models, 2023.

\bibitem[Gower et~al.(2021)Gower, Defazio, and Rabbat]{spsmt}
Robert~M. Gower, Aaron Defazio, and Michael~G. Rabbat.
\newblock Stochastic polyak stepsize with a moving target.
\newblock \emph{CoRR}, abs/2106.11851, 2021.
\newblock URL \url{https://arxiv.org/abs/2106.11851}.

\bibitem[Gower et~al.(2022)Gower, Blondel, Gazagnadou, and
  Pedregosa]{gower2022cutting}
Robert~M Gower, Mathieu Blondel, Nidham Gazagnadou, and Fabian Pedregosa.
\newblock Cutting some slack for sgd with adaptive polyak stepsizes.
\newblock \emph{arXiv preprint arXiv:2202.12328}, 2022.

\bibitem[Hastie et~al.(2009)Hastie, Tibshirani, and Friedman]{hastie}
Trevor Hastie, Robert Tibshirani, and Jerome Friedman.
\newblock \emph{The elements of statistical learning: data mining, inference
  and prediction}.
\newblock Springer, 2 edition, 2009.
\newblock URL \url{http://www-stat.stanford.edu/~tibs/ElemStatLearn/}.

\bibitem[Hazan and Kakade(2019)]{hazan_polyak}
Elad Hazan and Sham Kakade.
\newblock Revisiting the polyak step size, 2019.
\newblock URL \url{https://arxiv.org/abs/1905.00313}.

\bibitem[He et~al.(2016)He, Zhang, Ren, and Sun]{resnet}
Kaiming He, Xiangyu Zhang, Shaoqing Ren, and Jian Sun.
\newblock Deep residual learning for image recognition.
\newblock In \emph{2016 IEEE Conference on Computer Vision and Pattern
  Recognition (CVPR)}, pages 770--778, 2016.
\newblock \doi{10.1109/CVPR.2016.90}.

\bibitem[Johnson and Zhang(2013)]{VR_zhang}
Rie Johnson and Tong Zhang.
\newblock Accelerating stochastic gradient descent using predictive variance
  reduction.
\newblock In C.J. Burges, L.~Bottou, M.~Welling, Z.~Ghahramani, and K.Q.
  Weinberger, editors, \emph{Advances in Neural Information Processing
  Systems}, volume~26. Curran Associates, Inc., 2013.
\newblock URL
  \url{https://proceedings.neurips.cc/paper_files/paper/2013/file/ac1dd209cbcc5e5d1c6e28598e8cbbe8-Paper.pdf}.

\bibitem[Kairouz et~al.(2021)Kairouz, McMahan, Avent, Bellet, Bennis,
  Nitin~Bhagoji, Bonawitz, Charles, Cormode, Cummings, D’Oliveira, Eichner,
  El~Rouayheb, Evans, Gardner, Garrett, Gasc\'{o}n, Ghazi, Gibbons, Gruteser,
  Harchaoui, He, He, Huo, Hutchinson, Hsu, Jaggi, Javidi, Joshi, Khodak,
  Konecn\'{y}, Korolova, Koushanfar, Koyejo, Lepoint, Liu, Mittal, Mohri, Nock,
  \"{O}zg\"{u}r, Pagh, Qi, Ramage, Raskar, Raykova, Song, Song, Stich, Sun,
  Suresh, Tram\`{e}r, Vepakomma, Wang, Xiong, Xu, Yang, Yu, Yu, and Zhao]{FL}
Peter Kairouz, H.~Brendan McMahan, Brendan Avent, Aur\'{e}lien Bellet, Mehdi
  Bennis, Arjun Nitin~Bhagoji, Kallista Bonawitz, Zachary Charles, Graham
  Cormode, Rachel Cummings, Rafael G.~L. D’Oliveira, Hubert Eichner, Salim
  El~Rouayheb, David Evans, Josh Gardner, Zachary Garrett, Adri\`{a}
  Gasc\'{o}n, Badih Ghazi, Phillip~B. Gibbons, Marco Gruteser, Zaid Harchaoui,
  Chaoyang He, Lie He, Zhouyuan Huo, Ben Hutchinson, Justin Hsu, Martin Jaggi,
  Tara Javidi, Gauri Joshi, Mikhail Khodak, Jakub Konecn\'{y}, Aleksandra
  Korolova, Farinaz Koushanfar, Sanmi Koyejo, Tancr\`{e}de Lepoint, Yang Liu,
  Prateek Mittal, Mehryar Mohri, Richard Nock, Ayfer \"{O}zg\"{u}r, Rasmus
  Pagh, Hang Qi, Daniel Ramage, Ramesh Raskar, Mariana Raykova, Dawn Song,
  Weikang Song, Sebastian~U. Stich, Ziteng Sun, Ananda~Theertha Suresh, Florian
  Tram\`{e}r, Praneeth Vepakomma, Jianyu Wang, Li~Xiong, Zheng Xu, Qiang Yang,
  Felix~X. Yu, Han Yu, and Sen Zhao.
\newblock Advances and open problems in federated learning.
\newblock \emph{Found. Trends Mach. Learn.}, 14\penalty0 (1–2):\penalty0
  1–210, jun 2021.
\newblock ISSN 1935-8237.
\newblock \doi{10.1561/2200000083}.
\newblock URL \url{https://doi.org/10.1561/2200000083}.

\bibitem[Karimireddy et~al.(2020)Karimireddy, Kale, Mohri, Reddi, Stich, and
  Suresh]{scaffold}
Sai~Praneeth Karimireddy, Satyen Kale, Mehryar Mohri, Sashank Reddi, Sebastian
  Stich, and Ananda~Theertha Suresh.
\newblock {SCAFFOLD}: Stochastic controlled averaging for federated learning.
\newblock In Hal~Daumé III and Aarti Singh, editors, \emph{Proceedings of the
  37th International Conference on Machine Learning}, volume 119 of
  \emph{Proceedings of Machine Learning Research}, pages 5132--5143. PMLR,
  13--18 Jul 2020.
\newblock URL \url{https://proceedings.mlr.press/v119/karimireddy20a.html}.

\bibitem[Kavis et~al.(2022)Kavis, Skoulakis, Antonakopoulos, Dadi, and
  Cevher]{svrg-nonconvex}
Ali Kavis, Stratis Skoulakis, Kimon Antonakopoulos, Leello~Tadesse Dadi, and
  Volkan Cevher.
\newblock Adaptive stochastic variance reduction for non-convex finite-sum
  minimization.
\newblock In S.~Koyejo, S.~Mohamed, A.~Agarwal, D.~Belgrave, K.~Cho, and A.~Oh,
  editors, \emph{Advances in Neural Information Processing Systems}, volume~35,
  pages 23524--23538. Curran Associates, Inc., 2022.
\newblock URL
  \url{https://proceedings.neurips.cc/paper_files/paper/2022/file/94f625dcdec313cd432d65f96fcc51c8-Paper-Conference.pdf}.

\bibitem[Kingma and Ba(2015)]{adam}
Diederik~P. Kingma and Jimmy Ba.
\newblock Adam: {A} method for stochastic optimization.
\newblock In Yoshua Bengio and Yann LeCun, editors, \emph{3rd International
  Conference on Learning Representations, {ICLR} 2015, San Diego, CA, USA, May
  7-9, 2015, Conference Track Proceedings}, 2015.
\newblock URL \url{http://arxiv.org/abs/1412.6980}.

\bibitem[Kovalev et~al.(2020)Kovalev, Horv{\'a}th, and
  Richt{\'a}rik]{loopless_svrg}
Dmitry Kovalev, Samuel Horv{\'a}th, and Peter Richt{\'a}rik.
\newblock Don’t jump through hoops and remove those loops: Svrg and katyusha
  are better without the outer loop.
\newblock In Aryeh Kontorovich and Gergely Neu, editors, \emph{Proceedings of
  the 31st International Conference on Algorithmic Learning Theory}, volume 117
  of \emph{Proceedings of Machine Learning Research}, pages 451--467. PMLR, 08
  Feb--11 Feb 2020.
\newblock URL \url{https://proceedings.mlr.press/v117/kovalev20a.html}.

\bibitem[Krizhevsky et~al.({\natexlab{a}})Krizhevsky, Nair, and
  Hinton]{cifar10}
Alex Krizhevsky, Vinod Nair, and Geoffrey Hinton.
\newblock Cifar-10 (canadian institute for advanced research).
\newblock {\natexlab{a}}.
\newblock URL \url{http://www.cs.toronto.edu/~kriz/cifar.html}.

\bibitem[Krizhevsky et~al.({\natexlab{b}})Krizhevsky, Nair, and
  Hinton]{cifar100}
Alex Krizhevsky, Vinod Nair, and Geoffrey Hinton.
\newblock Cifar-100 (canadian institute for advanced research).
\newblock {\natexlab{b}}.
\newblock URL \url{http://www.cs.toronto.edu/~kriz/cifar.html}.

\bibitem[Kunstner et~al.(2023)Kunstner, Portella, Schmidt, and
  Harvey]{multidimls}
Frederik Kunstner, Victor~S. Portella, Mark Schmidt, and Nick Harvey.
\newblock Searching for optimal per-coordinate step-sizes with multidimensional
  backtracking, 2023.

\bibitem[Li et~al.(2022)Li, Swartworth, Tak{\'a}{\v{c}}, Needell, and
  Gower]{sp2}
Shuang Li, William~J Swartworth, Martin Tak{\'a}{\v{c}}, Deanna Needell, and
  Robert~M Gower.
\newblock Sp2: A second order stochastic polyak method.
\newblock \emph{arXiv preprint arXiv:2207.08171}, 2022.

\bibitem[Li et~al.(2020)Li, Sahu, Zaheer, Sanjabi, Talwalkar, and
  Smith]{fedprox}
Tian Li, Anit~Kumar Sahu, Manzil Zaheer, Maziar Sanjabi, Ameet Talwalkar, and
  Virginia Smith.
\newblock Federated optimization in heterogeneous networks.
\newblock In Inderjit~S. Dhillon, Dimitris~S. Papailiopoulos, and Vivienne Sze,
  editors, \emph{Proceedings of Machine Learning and Systems 2020, MLSys 2020,
  Austin, TX, USA, March 2-4, 2020}. mlsys.org, 2020.
\newblock URL \url{https://proceedings.mlsys.org/book/316.pdf}.

\bibitem[Li et~al.(2021)Li, Bao, Zhang, and Richtarik]{page}
Zhize Li, Hongyan Bao, Xiangliang Zhang, and Peter Richtarik.
\newblock Page: A simple and optimal probabilistic gradient estimator for
  nonconvex optimization.
\newblock In Marina Meila and Tong Zhang, editors, \emph{Proceedings of the
  38th International Conference on Machine Learning}, volume 139 of
  \emph{Proceedings of Machine Learning Research}, pages 6286--6295. PMLR,
  18--24 Jul 2021.
\newblock URL \url{https://proceedings.mlr.press/v139/li21a.html}.

\bibitem[Loizou et~al.(2021)Loizou, Vaswani, Laradji, and
  Lacoste{-}Julien]{sps}
Nicolas Loizou, Sharan Vaswani, Issam~Hadj Laradji, and Simon Lacoste{-}Julien.
\newblock Stochastic polyak step-size for {SGD:} an adaptive learning rate for
  fast convergence.
\newblock In Arindam Banerjee and Kenji Fukumizu, editors, \emph{The 24th
  International Conference on Artificial Intelligence and Statistics, {AISTATS}
  2021, April 13-15, 2021, Virtual Event}, volume 130 of \emph{Proceedings of
  Machine Learning Research}, pages 1306--1314. {PMLR}, 2021.
\newblock URL \url{http://proceedings.mlr.press/v130/loizou21a.html}.

\bibitem[Loshchilov and Hutter(2017)]{warmupstepsize}
Ilya Loshchilov and Frank Hutter.
\newblock {SGDR}: Stochastic gradient descent with warm restarts.
\newblock In \emph{International Conference on Learning Representations}, 2017.
\newblock URL \url{https://openreview.net/forum?id=Skq89Scxx}.

\bibitem[Mishchenko et~al.(2022)Mishchenko, Malinovsky, Stich, and
  Richt{\'a}rik]{proxskip}
Konstantin Mishchenko, Grigory Malinovsky, Sebastian Stich, and Peter
  Richt{\'a}rik.
\newblock Proxskip: Yes! local gradient steps provably lead to communication
  acceleration! finally!
\newblock In \emph{International Conference on Machine Learning}, pages
  15750--15769. PMLR, 2022.

\bibitem[Mukkamala and Hein(2017)]{log-AdaGrad}
Mahesh~Chandra Mukkamala and Matthias Hein.
\newblock Variants of {RMSP}rop and {A}dagrad with logarithmic regret bounds.
\newblock In Doina Precup and Yee~Whye Teh, editors, \emph{Proceedings of the
  34th International Conference on Machine Learning}, volume~70 of
  \emph{Proceedings of Machine Learning Research}, pages 2545--2553. PMLR,
  06--11 Aug 2017.
\newblock URL \url{https://proceedings.mlr.press/v70/mukkamala17a.html}.

\bibitem[Nedi{\'c} and Bertsekas(2001)]{increment-polyak}
Angelia Nedi{\'c} and Dimitri Bertsekas.
\newblock \emph{Convergence Rate of Incremental Subgradient Algorithms}, pages
  223--264.
\newblock Springer US, Boston, MA, 2001.
\newblock ISBN 978-1-4757-6594-6.
\newblock \doi{10.1007/978-1-4757-6594-6{\_}11}.
\newblock URL \url{https://doi.org/10.1007/978-1-4757-6594-6_11}.

\bibitem[Nesterov(2014)]{nesterov_book}
Yurii Nesterov.
\newblock \emph{Introductory Lectures on Convex Optimization: A Basic Course}.
\newblock Springer Publishing Company, Incorporated, 1 edition, 2014.
\newblock ISBN 1461346916.

\bibitem[Nguyen et~al.(2017)Nguyen, Liu, Scheinberg, and
  Tak{\'a}{\v{c}}]{sarah}
Lam~M. Nguyen, Jie Liu, Katya Scheinberg, and Martin Tak{\'a}{\v{c}}.
\newblock {SARAH}: A novel method for machine learning problems using
  stochastic recursive gradient.
\newblock In Doina Precup and Yee~Whye Teh, editors, \emph{Proceedings of the
  34th International Conference on Machine Learning}, volume~70 of
  \emph{Proceedings of Machine Learning Research}, pages 2613--2621. PMLR,
  06--11 Aug 2017.
\newblock URL \url{https://proceedings.mlr.press/v70/nguyen17b.html}.

\bibitem[Nocedal and Wright(2006)]{num_opt_book}
Jorge Nocedal and Stephen~J. Wright.
\newblock \emph{Numerical Optimization}.
\newblock Springer, New York, NY, USA, 2e edition, 2006.

\bibitem[Oberman and Prazeres(2019)]{oberman2019}
Adam~M Oberman and Mariana Prazeres.
\newblock Stochastic gradient descent with polyak's learning rate.
\newblock \emph{arXiv preprint arXiv:1903.08688}, 2019.

\bibitem[Orabona and P{\'a}l(2015)]{adanorm-2015}
Francesco Orabona and D{\'a}vid P{\'a}l.
\newblock Scale-free algorithms for online linear optimization.
\newblock In Kamalika Chaudhuri, CLAUDIO GENTILE, and Sandra Zilles, editors,
  \emph{Algorithmic Learning Theory}, pages 287--301, Cham, 2015. Springer
  International Publishing.
\newblock ISBN 978-3-319-24486-0.

\bibitem[Orvieto et~al.(2022)Orvieto, Lacoste-Julien, and Loizou]{decsps}
Antonio Orvieto, Simon Lacoste-Julien, and Nicolas Loizou.
\newblock Dynamics of sgd with stochastic polyak stepsizes: Truly adaptive
  variants and convergence to exact solution.
\newblock In S.~Koyejo, S.~Mohamed, A.~Agarwal, D.~Belgrave, K.~Cho, and A.~Oh,
  editors, \emph{Advances in Neural Information Processing Systems}, volume~35,
  pages 26943--26954. Curran Associates, Inc., 2022.
\newblock URL
  \url{https://proceedings.neurips.cc/paper_files/paper/2022/file/ac662d74829e4407ce1d126477f4a03a-Paper-Conference.pdf}.

\bibitem[Polyak(1987)]{polyak-book}
B.~T. Polyak.
\newblock \emph{Introduction to optimization}.
\newblock Translations series in mathematics and engineering. Optimization
  Software, Publications Division, New York, 1987.
\newblock ISBN 0911575146; 9780911575149.

\bibitem[Robbins and Monro(1951)]{sgd_robbin}
Herbert Robbins and Sutton Monro.
\newblock {A Stochastic Approximation Method}.
\newblock \emph{The Annals of Mathematical Statistics}, 22\penalty0
  (3):\penalty0 400 -- 407, 1951.
\newblock \doi{10.1214/aoms/1177729586}.
\newblock URL \url{https://doi.org/10.1214/aoms/1177729586}.

\bibitem[Rolinek and Martius(2018)]{L4}
Michal Rolinek and Georg Martius.
\newblock L4: Practical loss-based stepsize adaptation for deep learning.
\newblock In \emph{Advances in Neural Information Processing Systems 31
  (NeurIPS 2018)}, pages 6434--6444. Curran Associates, Inc., 2018.
\newblock URL
  \url{http://papers.nips.cc/paper/7879-l4-practical-loss-based-stepsize-adaptation-for-deep-learning.pdf}.

\bibitem[Schmidt et~al.(2017)Schmidt, Le~Roux, and Bach]{sag}
Mark Schmidt, Nicolas Le~Roux, and Francis Bach.
\newblock Minimizing finite sums with the stochastic average gradient.
\newblock \emph{Mathematical Programming}, 162\penalty0 (1):\penalty0 83--112,
  2017.
\newblock \doi{10.1007/s10107-016-1030-6}.
\newblock URL \url{https://doi.org/10.1007/s10107-016-1030-6}.

\bibitem[Streeter and McMahan(2010)]{adanorm-2010}
Matthew~J. Streeter and H.~Brendan McMahan.
\newblock Less regret via online conditioning.
\newblock \emph{CoRR}, abs/1002.4862, 2010.
\newblock URL \url{http://arxiv.org/abs/1002.4862}.

\bibitem[Vaswani et~al.(2019)Vaswani, Mishkin, Laradji, Schmidt, Gidel, and
  Lacoste{-}Julien]{sls}
Sharan Vaswani, Aaron Mishkin, Issam~H. Laradji, Mark Schmidt, Gauthier Gidel,
  and Simon Lacoste{-}Julien.
\newblock Painless stochastic gradient: Interpolation, line-search, and
  convergence rates.
\newblock In Hanna~M. Wallach, Hugo Larochelle, Alina Beygelzimer, Florence
  d'Alch{\'{e}}{-}Buc, Emily~B. Fox, and Roman Garnett, editors, \emph{Advances
  in Neural Information Processing Systems 32: Annual Conference on Neural
  Information Processing Systems 2019, NeurIPS 2019, December 8-14, 2019,
  Vancouver, BC, Canada}, pages 3727--3740, 2019.
\newblock URL
  \url{https://proceedings.neurips.cc/paper/2019/hash/2557911c1bf75c2b643afb4ecbfc8ec2-Abstract.html}.

\bibitem[Vaswani et~al.(2020)Vaswani, Laradji, Kunstner, Meng, Schmidt, and
  Lacoste-Julien]{adagrad+ls}
Sharan Vaswani, Issam Laradji, Frederik Kunstner, Si~Yi Meng, Mark Schmidt, and
  Simon Lacoste-Julien.
\newblock Adaptive gradient methods converge faster with over-parameterization
  (but you should do a line-search).
\newblock \emph{arXiv preprint arXiv:2006.06835}, 2020.

\bibitem[Ward et~al.(2019)Ward, Wu, and Bottou]{wald}
Rachel Ward, Xiaoxia Wu, and Leon Bottou.
\newblock {A}da{G}rad stepsizes: Sharp convergence over nonconvex landscapes.
\newblock In Kamalika Chaudhuri and Ruslan Salakhutdinov, editors,
  \emph{Proceedings of the 36th International Conference on Machine Learning},
  volume~97 of \emph{Proceedings of Machine Learning Research}, pages
  6677--6686. PMLR, 09--15 Jun 2019.
\newblock URL \url{https://proceedings.mlr.press/v97/ward19a.html}.

\bibitem[Xie et~al.(2020)Xie, Wu, and Ward]{linear-AdaGrad}
Yuege Xie, Xiaoxia Wu, and Rachel Ward.
\newblock Linear convergence of adaptive stochastic gradient descent.
\newblock In Silvia Chiappa and Roberto Calandra, editors, \emph{Proceedings of
  the Twenty Third International Conference on Artificial Intelligence and
  Statistics}, volume 108 of \emph{Proceedings of Machine Learning Research},
  pages 1475--1485. PMLR, 26--28 Aug 2020.
\newblock URL \url{https://proceedings.mlr.press/v108/xie20a.html}.

\end{thebibliography}
}
\appendix
\numberwithin{equation}{section}
\numberwithin{figure}{section}
\numberwithin{table}{section}

\newpage
\small
{\Huge\textbf{Appendix}}
\vspace{0.5cm}
\section{Technical Preliminaries}
\label{sec:pf_lm}

\subsection{Basic Definitions}
We use the following definitions throughout the paper.
\begin{definition}[convexity]
A differentiable function $f:\R^d\to\R$ is convex if\;$\forall\;\xx,\yy\in\R^d$,

\begin{equation}
    f(\yy)\ge f(\xx)+\lin{\nabla f(\xx),\yy-\xx}\;.
    \label{df:convex}
    \end{equation}
\end{definition} 

\begin{definition}[strong convexity]
A differentiable function $f:\R^d\to\mathbb{R}$ is $\mu$-strongly convex if\;$\forall\; \xx,\yy\in\R^d$,
\begin{equation}
    f(\yy)\ge f(\xx)+\lin{\nabla f(\xx),\yy-\xx}+\frac{\mu}{2}||\xx-\yy||^2\;.
    \label{df:stconvex}
\end{equation}
\end{definition} 

\begin{definition}[$L$-smooth]
Let function $f:\R^d\to\mathbb{R}$ be differentiable. $f$ is smooth if there exists $L>0$ such that $\forall\;\xx,\yy\in\R^d$,
\begin{equation}
    ||\nabla f(\xx)-\nabla f(\yy)||\le L||\xx-\yy||\;.
\end{equation}
\label{df:smooth}
\end{definition}

\subsection{Useful Lemmas}
We frequently use the following helpful lemmas for the proof. 
\begin{lemma}[\citet{nesterov_book}, Lemma 1.2.3]
    Definition~\ref{df:smooth} implies that there exists a quadratic upper bound on f:
\begin{equation}
    f(\yy)\le f(\xx)+\lin{\nabla f(\xx),\yy-\xx}|+\frac{L}{2}||\yy-\xx||^2 \;, \forall\xx,\yy\in\R^d\;. \label{df:smooth:quad}
\end{equation}
\end{lemma}

\begin{lemma}[\citet{nesterov_book}, Theorem 2.1.5]
    If a convex function $f$ satisfies \eqref{df:smooth:quad}, then it holds that:
\begin{equation}
    f(\yy)\ge f(\xx)+\lin{\nabla f(\xx),\yy-\xx}|+\frac{1}{2L}||\nabla f(\yy)-\nabla f(\xx)||^2 \;, \forall\xx,\yy\in\R^d \label{eq:smooth+convex:property}.
\end{equation}
\end{lemma}

\begin{lemma}[\citet{wald}]
For any non-negative sequence $a_0,...,a_T$, the following holds:
\begin{equation}
\sqrt{\sum_{t=0}^Ta_t}\le\sum_{t=0}^T\frac{a_t}{\sqrt{\sum_{i=0}^ta_i}}\le2\sqrt{\sum_{t=0}^Ta_t}\;.
    \label{eq:wald}
\end{equation}
If $a_0\ge1$, then the following holds:
\begin{equation}
    \sum_{t=0}^T\frac{a_t}{\sum_{i=0}^ta_i}\le\log(\sum_{t=0}^Ta_t)+1\;.
    \label{eq:wald2}
\end{equation}
\end{lemma}
\begin{proof}
    To show equation \eqref{eq:wald}, we proceed with the proof by induction. For $t=0$, \eqref{eq:wald} holds trivially since $\sqrt{a_0}\le\sqrt{a_0}\le2\sqrt{a_0}$. 
    Assume equation \eqref{eq:wald} holds for $T-1$. For RHS, we have:
    \begin{equation}
    \begin{split}
        \sum_{t=0}^{T-1}\frac{a_t}{\sum_{i=0}^ta_i}+\frac{a_T}{\sqrt{\sum_{i=0}^T a_i}}&\le2\sqrt{\sum_{t=0}^{T-1}a_t}+\frac{a_T}{\sqrt{\sum_{i=0}^T a_i}}\\
        &=2\sqrt{\sum_{t=0}^{T}a_t-a_T}+\frac{a_T}{\sqrt{\sum_{t=0}^T a_t}} \\ 
        &\le 2\sqrt{\sum_{t=0}^Ta_t}\;.
    \end{split}
    \end{equation}
    where the last inequality is due to the fact that $2\sqrt{x-y}+\frac{y}{\sqrt{x}}\le2\sqrt{x}$ for any $x\ge y\ge0$. For LHS, we have: 
    \begin{equation}
    \begin{split}
        \sum_{t=0}^{T-1}\frac{a_t}{\sum_{i=0}^ta_i}+\frac{a_T}{\sqrt{\sum_{i=0}^T a_i}}&\ge\sqrt{\sum_{t=0}^{T-1}a_t}+\frac{a_T}{\sqrt{\sum_{i=0}^T a_i}}\\
        &=\sqrt{\sum_{t=0}^{T}a_t-a_T}+\frac{a_T}{\sqrt{\sum_{t=0}^T a_t}} \\ 
        &\ge \sqrt{\sum_{t=0}^Ta_t}\;.
    \end{split}
    \end{equation}
    where the last inequality is due to the fact that $\sqrt{x-y}+\frac{y}{\sqrt{x}}\ge\sqrt{x}$ for any $x\ge y\ge0$.
    
    We next show equation \eqref{eq:wald2} by induction. For $t=0$, equation \eqref{eq:wald2} trivially holds since $1\le\log(a_0)+1$. Assume \eqref{eq:wald2} holds for $T-1$, we have:
    \begin{equation}
    \begin{split}
        \sum_{t=0}^{T}\frac{a_t}{\sum_{i=0}^t a_i}&\le\log(\sum_{t=0}^{T-1}a_t)+1+\frac{a_T}{\sum_{i=0}^T a_i} \\
        &\le\log(\sum_{t=0}^Ta_t)+1\;.
    \end{split}
    \end{equation}
    where the last inequality is due to the fact that $\log(x-y)+\frac{y}{x}\le\log(x)$ for any $x\ge y\ge0$ since $e^{\frac{y}{x}}\le\frac{1+\frac{y}{x}}{1-\frac{y^2}{x^2}}$.
\end{proof}

\begin{lemma}[{\citet[Lemma 5]{ada-svrg}}]
If $x^2\le a(x+b)$ for $a\ge 0$ and $b\ge 0$, then it holds that:
\begin{equation}
    x\le a+\sqrt{ab}\;.
    \label{eq:simple_quadratic}
\end{equation}
\label{lemma:simple_quadratic}
\end{lemma}

The following Lemma is an extension of Lemma 5 in \cite{decsps}.
\begin{lemma}
Let $z_{t+1}\le(1-a\eta_t)z_t+\eta_tb$ and $z_t\ge 0$ where $a>0$, $b>0$ and $\eta_t>0, \eta_{t+1}\le\eta_{t},\;\forall t\ge 0$. It holds that:
\begin{equation}
    z_t\le\max\{\frac{b}{a},z_0,\eta_0 b\},\;\forall t\ge 0\;.
\end{equation}
\begin{proof}
    Since $\eta_t$ is non-increasing, $1-a\eta_t\le 0$ is non-decreasing. For any $t\ge0$ such that
    $1-a\eta_t\le 0$, we have $z_{t+1}\le\eta_tb\le\eta_0b$. If $1-a\eta_t\le 0$ for all $t\ge 0$, then the proof is done. Otherwise, let us assume there exists a first index $j$ such that $1-a\eta_j> 0$ and we have $z_j\le\max\{z_0,\eta_0b\}:=\Tilde{z}_0$.  We proceed with the proof starting with the index $j$ by induction.
    For $t=j$, the lemma trivially holds. Let us assume $z_t\le\max\{\frac{b}{a},\Tilde{z}_0\}$ for $t>j$. If $\frac{b}{a}\ge \Tilde{z}_0$, then by induction, we have:
    \begin{equation}
        z_{t+1}\le (1-a\eta_t)\frac{b}{a}+\eta_tb=\frac{b}{a}\;.
    \end{equation}
    If instead $\frac{b}{a}\le \Tilde{z}_0$, then by induction, we have:
    \begin{equation}
        z_{t+1}\le (1-a\eta_t)\Tilde{z}_0+\eta_tb=\Tilde{z}_0-\eta_t(a\Tilde{z}_0-b)\le \Tilde{z}_0\;.
    \end{equation}
Combining the above cases concludes the proof.
\end{proof}
\label{lemma:decsps}
\end{lemma}

The following lemma is commonly used in the works on Polyak stepsize~\cite{sps,hazan_polyak}.
\begin{lemma}
    Suppose a function $f$ is $L$-smooth and $\mu$-strongly convex, then the following holds:
    \begin{equation}
        \frac{1}{2L}\le\frac{f(\xx)-f^\star}{||\nabla f(\xx)||^2}\le\frac{1}{2\mu}\;.
        \label{eq:step_bound}
    \end{equation}
\end{lemma}

The following lemma provides upper and lower bounds for the stepsize of AdaSPS.
\begin{lemma}
Suppose each $f_i$ is $L$-smooth, then the stepsize of AdaSPS \eqref{AdaSPS} satisfies:
\begin{equation}
    \frac{1}{2c_pL}\frac{1}{\sqrt{\sum_{s=0}^t f_{i_s}(\xx_s)-\ell_{i_s}^\star}}\le\eta_t\le\frac{f_{i_t}(\xx_t)-\ell_{i_t}^\star}{c_p||\nabla f_{i_t}(\xx_t)||^2}\frac{1}{\sqrt{\sum_{s=0}^t f_{i_s}(\xx_s)-\ell_{i_s}^\star}}\;.
    \label{eq:step_bound_adasps}
\end{equation}
\begin{proof}
The upper bound follows from the definition of the stepsize \eqref{AdaSPS}. To prove the lower bound, we note that the stepsize \eqref{AdaSPS} is composed of two parts where the first component $\frac{f_{i_s}(\xx_s)-\ell_{i_s}^\star}{c_p||\nabla f_{i_s}(\xx_s)||^2}\ge\frac{1}{2c_pL}$ for all $0\le s\le t$ due to \eqref{eq:step_bound}, and the second component is always decreasing. Finally, recall that $\eta_{-1}=+\infty$ and thus the proof is completed.
\end{proof}
\label{lemma:step_bound_adasps}
\end{lemma}

The following lemma provides upper and lower bounds for the stepsize of AdaSLS. We refer to Appendix \ref{sec:line-search} for details of the line-search procedure.

\begin{lemma}
Suppose each $f_i$ is $L$-smooth, then the stepsize of AdaSLS \eqref{AdaSLS} satisfies:
\begin{equation}
    \min\Bigl\{\frac{1-\rho}{L},\gamma_{\max}\Bigr\}\frac{1}{c_l\sqrt{\sum_{s=0}^t\gamma_s||\nabla f_{i_s}(\xx_s)||^2}}
    \le\eta_t\le\frac{\gamma_t}{c_l\sqrt{\sum_{s=0}^t\gamma_s||\nabla f_{i_s}(\xx_s)||^2}}\;.
    \label{eq:step_bound_adasls}
\end{equation}
\begin{proof}
    The upper bound is due to the definition of the stepsize \eqref{AdaSLS}. 
    We next prove the lower bound. From the smoothness definition, the following holds for all $\gamma_t$:
    \begin{equation}
    \begin{split}
        f_{i_t}(\xx_t-\gamma_t\nabla f_{i_t}(\xx_t))&\stackrel{\eqref{df:smooth:quad}}{\le} f_{i_t}(\xx_t)-\gamma_t||\nabla f_{i_t}(\xx_t)||^2+\frac{L}{2}\gamma_t^2||\nabla f_{i_t}(\xx_t)||^2\;.
    \end{split}
    \end{equation}
    For any $0<\gamma_t\le\frac{2(1-\rho)}{L}$, we have:
    \begin{equation}
        f_{i_t}(\xx_t-\gamma_t\nabla f_{i_t}(
\xx_t))\le f_{i_t}(\xx_t)-\rho\gamma_t||\nabla f_{i_t}(\xx_t)||^2\;,
    \end{equation}
    which satisfies the line-search condition~\eqref{eq:Armijo}.
    From the procedure of Backtracking line-search (Alg.~\ref{Algorithm:Armijo}), if $\gamma_{\max}\le\frac{1-\rho}{L}$, then $\gamma_t=\gamma_{\max}$ is accepted. Otherwise, since we require the decreasing factor $\beta$ to be no smaller than $\frac{1}{2}$, we must have $\gamma_{t}\ge\frac{2(1-\rho)}{2L}$. Therefore, $\gamma_t$ is always lower bounded by $\min\{\frac{1-\rho}{L},\gamma_{\max}\}$. The second component of AdaSLS is always decreasing, and recall that $\eta_{-1}=+\infty$. The proof is thus completed.
\end{proof}
\label{lemma:step_bound_adasls}
\end{lemma}

\section{Proofs of main results}
\label{sec:pf_main}
\subsection{Proof of Theorem \ref{thm:convex}}
\begin{proof}
    We follow a common proof routine for the general convex optimization~\cite{AdaGrad,ada-svrg,decsps}.
    Using the update rule of PSGD \eqref{eq:PSGD}, we have: 
    \begin{equation}
    \begin{split}
        ||\xx_{t+1}-\xx^\star||^2&=||\Pi_{\cX}(\xx_t-\eta_t\nabla f_{i_t}(\xx_t))-\Pi_{\cX}(\xx^\star)||^2 \\ 
        &\le||\xx_t-\eta_t\nabla f_{i_t}(\xx_t)-\xx^\star||^2 \\ 
        &=||\xx_t-\xx^\star||^2-2\eta_t\langle\nabla f_{i_t}(\xx_t),\xx_t-\xx^\star\rangle+\eta_t^2||\nabla f_{i_t}(\xx_t)||^2\;.
    \end{split}
    \end{equation}
Dividing by $2\eta_t$ and rearranging gives:
\begin{equation}
\begin{split}
&\langle\nabla f_{i_t}(\xx_t),\xx_t-\xx^\star\rangle \\&\le\frac{||\xx_t-\xx^\star||^2}{2\eta_t}-\frac{||\xx_{t+1}-\xx^\star||^2}{2\eta_t}+\frac{\eta_t}{2}||\nabla f_{i_t}(\xx_t)||^2 \\ 
&= \frac{||\xx_t-\xx^\star||^2}{2\eta_t}-\frac{||\xx_{t+1}-\xx^\star||^2}{2\eta_{t+1}}+\frac{||\xx_{t+1}-\xx^\star||^2}{2\eta_{t+1}}-\frac{||\xx_{t+1}-\xx^\star||^2}{2\eta_t}+\frac{\eta_t}{2}||\nabla f_{i_t}(\xx_t)||^2\;.
\end{split}
\end{equation}
Summing from $t=0$ to $t=T-1$, we get: 
\begin{equation}
\begin{split}
&\sum_{t=0}^{T-1}\langle\nabla f_{i_t}(\xx_t),\xx_t-\xx^\star\rangle 
\\&\le\sum_{t=0}^{T-1}\frac{||\xx_t-\xx^\star||^2}{2\eta_t}-\frac{||\xx_{t+1}-\xx^\star||^2}{2\eta_{t+1}}+\frac{||\xx_{t+1}-\xx^\star||^2}{2\eta_{t+1}}-\frac{||\xx_{t+1}-\xx^\star||^2}{2\eta_t}+\frac{\eta_t}{2}||\nabla f_{i_t}(\xx_t)||^2
\\&\le\frac{||\xx_0-\xx^\star||^2}{2\eta_0}-\frac{||\xx_T-\xx^\star||^2}{2\eta_{T}}+\frac{||\xx_T-\xx^\star||^2}{2\eta_{T}}-\frac{||\xx_T-\xx^\star||^2}{2\eta_{T-1}}+\sum_{t=0}^{T-2}(\frac{1}{2\eta_{t+1}}-\frac{1}{2\eta_t})D^2+\sum_{t=0}^{T-1}\frac{\eta_t}{2}||\nabla f_{i_t}(\xx_t)||^2
\\&\le
\frac{||\xx_0-\xx^\star||^2}{2\eta_0}-\frac{||\xx_T-\xx^\star||^2}{2\eta_{T}}+\frac{||\xx_T-\xx^\star||^2}{2\eta_{T}}+\frac{D^2}{2\eta_{T-1}}+\sum_{t=0}^{T-1}\frac{\eta_t}{2}||\nabla f_{i_t}(\xx_t)||^2 \\ 
&=\frac{||\xx_0-\xx^\star||^2}{2\eta_0}+\frac{D^2}{2\eta_{T-1}}+\sum_{t=0}^{T-1}\frac{\eta_t}{2}||\nabla f_{i_t}(\xx_t)||^2\;,
\label{eq:thm:convex:1}
\end{split}
\end{equation}
where in the second inequality, we use the decreasing property of the stepsize $\eta_t$ which guarantees $\frac{1}{2\eta_t}-\frac{1}{2\eta_{t-1}}\ge0$, and we use the fact that $||\xx_t-\xx^\star||^2\le D^2$ because of the projection step in \eqref{eq:PSGD}. For clarity, we next 
separate the proof for AdaSPS and AdaSLS. 

\textbf{AdaSPS:}
We upper bound the last two terms by using Lemma \ref{lemma:step_bound_adasps} and we obtain: 
\begin{equation}
    \sum_{t=0}^{T-1}\frac{\eta_t}{2}||\nabla f_{i_t}(\xx_t)||^2\stackrel{\eqref{eq:step_bound_adasps}}{\le}\sum_{t=0}^{T-1}\frac{f_{i_t}(\xx_t)-\ell_{i_t}^\star}{2c_p\sqrt{\sum_{s=0}^t f_{i_s}(\xx_s)-\ell_{i_s}^\star}}
    \stackrel{\eqref{eq:wald}}{\le}\frac{1}{c_p}\sqrt{\sum_{s=0}^{T-1} f_{i_s}(\xx_s)-\ell_{i_s}^\star}\;,
    \label{eq:thm:convex:2}
\end{equation}
and 
\begin{equation}
    \frac{D^2}{2\eta_{T-1}}\stackrel{\eqref{eq:step_bound_adasps}}{\le} c_pLD^2\sqrt{\sum_{s=0}^{T-1} f_{i_s}(\xx_s)-\ell_{i_s}^\star}\;.
    \label{eq:thm:convex:3}
\end{equation}
Using $\frac{||\xx_0-\xx^\star||^2}{2\eta_0}\le\frac{D^2}{2\eta_{T-1}}$ and plugging \eqref{eq:thm:convex:2} and \eqref{eq:thm:convex:3} back to \eqref{eq:thm:convex:1} gives:
\begin{equation}
    \sum_{t=0}^{T-1}\langle\nabla f_{i_t}(\xx_t),\xx_t-\xx^\star\rangle \le (2c_pLD^2+\frac{1}{c_p})\sqrt{\sum_{s=0}^{T-1} f_{i_s}(\xx_s)-\ell_{i_s}^\star}\;.
    \label{thm:vr:ref:1}
\end{equation}
Taking the expectation on both sides, we have:
\begin{equation}
\begin{split}
    \sum_{t=0}^{T-1}\E[\langle\nabla f(\xx_t),\xx_t-\xx^\star\rangle] &\le (2c_pLD^2+\frac{1}{c_p})\E\Bigl[\sqrt{\sum_{s=0}^{T-1} f_{i_s}(\xx_s)-\ell_{i_s}^\star}\Bigr] \\ 
    &=(2c_pLD^2+\frac{1}{c_p})\E\Bigl[\sqrt{\sum_{s=0}^{T-1} f_{i_s}(\xx_s)-f_{i_s}(\xx^\star)+f_{i_s}(\xx^\star)-\ell_{i_s}^\star}\Bigr]\;.
\end{split}
\end{equation}
Using the convexity assumption of $f$ and applying Jensen's inequality to the square root function, we get:
\begin{equation}
    \sum_{t=0}^{T-1}\E[f(\xx_t)-f^\star] \le (2c_pLD^2+\frac{1}{c_p})\sqrt{\sum_{s=0}^{T-1} \E[f(\xx_s)-f^\star]+\sigma_{f,B}^2+{\rm err}_{f,B}^2}\;,
\end{equation}
where ${\rm err}_{f,B}^2=\E_{i_s}[f_{i_s}^\star-\ell_{i_s}^\star]$. Let $\tau:=2c_pLD^2+\frac{1}{c_p}$. Taking the square gives: 
\begin{equation}
    (\sum_{t=0}^{T-1}\E[f(\xx_t)-f^\star])^2 \le \tau^2\Bigl(\sum_{t=0}^{T-1} \E[f(\xx_t)-f^\star]+T(\sigma_{f,B}^2+{\rm err}_{f,B}^2)\Bigr)\;.
\end{equation}
We next apply Lemma \ref{lemma:simple_quadratic} with $x=\sum_{t=0}^{T-1} \E[f(\xx_t)-f(\xx^\star)]$, $a=\tau^2$ and $b=T(\sigma_{f,B}^2+{\rm err}_{f,B}^2)$:
\begin{equation}
    \sum_{t=0}^{T-1}\E[f(\xx_t)-f^\star]\le\tau^2+\tau\sqrt{\sigma_{f,B}^2+{\rm err}_{f,B}^2}\sqrt{T}\;.
\end{equation}
We conclude by dividing both sides by $T$ and using Jensen's inequality:
\begin{equation}
    \E[f(\Bar{\xx}_{T})-f^\star]\le\frac{\sum_{t=0}^{T-1}\E[f(\xx_t)-f^\star]}{T}\le\frac{\tau^2}{T}+\frac{\tau\sqrt{\sigma_{f,B}^2+{\rm err}_{f,B}^2}}{\sqrt{T}}\;.
\end{equation}
where $\Bar{\xx}_T=\frac{1}{T}\sum_{t=0}^{T-1}\xx_t$.

\textbf{AdaSLS:} The proof is almost the same as AdaSPS. We omit procedures with the same proof reasons for simplicity. We first use Lemma \ref{lemma:step_bound_adasls} to obtain:
\begin{equation}
    \sum_{t=0}^{T-1}\frac{\eta_t}{2}||\nabla f_{i_t}(\xx_t)||^2\stackrel{\eqref{eq:step_bound_adasls}}{\le}\sum_{t=0}^{T-1}\frac{\gamma_t||\nabla f_{i_t}(\xx_t)||^2}{2c_l\sqrt{\sum_{s=0}^t\gamma_s||\nabla f_{i_s}(\xx_s)||^2}}\stackrel{\eqref{eq:wald}}{\le}\frac{1}{c_l}\sqrt{\sum_{s=0}^{T-1} \gamma_s||\nabla f_{i_s}(\xx_s)||^2}\;,
\end{equation}
and 
\begin{equation}
    \frac{D^2}{2\eta_{T-1}}\stackrel{\eqref{eq:step_bound_adasls}}{\le} \frac{c_l\sqrt{\sum_{s=0}^{T-1}\gamma_s||\nabla f_{i_s}(\xx_s)||^2}D^2}{2\min\Bigl\{\frac{1-\rho}{L},\gamma_{\max}\Bigr\}}=\frac{\max\{\frac{L}{1-\rho},\frac{1}{\gamma_{\max}}\}c_lD^2}{2}\sqrt{\sum_{s=0}^{T-1}\gamma_s||\nabla f_{i_s}(\xx_s)||^2}\;.
\end{equation}
Inequality \eqref{eq:thm:convex:1} can then be further bounded by:
\begin{equation}
\begin{split}
    \sum_{t=0}^{T-1}\langle\nabla f_{i_t}(\xx_t),\xx_t-\xx^\star\rangle &\le \Bigl(\max\Bigl\{\frac{L}{1-\rho},\frac{1}{\gamma_{\max}}\Bigr\}c_lD^2+\frac{1}{c_l}\Bigr)\sqrt{\sum_{s=0}^{T-1}\gamma_s||\nabla f_{i_s}(\xx_s)||^2} \\ 
    &\le \Bigr(\max\Bigl\{\frac{L}{(1-\rho)\sqrt{\rho}},\frac{1}{\gamma_{\max}\sqrt{\rho}}\Bigr\}c_lD^2+\frac{1}{c_l\sqrt{\rho}}\Bigr)\sqrt{\sum_{s=0}^{T-1}f_{i_s}(\xx_s)-f^\star_{i_s}}.
    \label{eq:thm:vr:ref:2}
\end{split}
\end{equation}
where we used line-search condition \eqref{eq:Armijo} and the fact that $f_{i_s}(\xx_t-\gamma_t\nabla f_{i_s}(
\xx_s))\ge f^\star_{i_s}$. 

Let $\tau:=\max\Bigl\{\frac{L}{(1-\rho)\sqrt{\rho}},\frac{1}{\gamma_{\max}\sqrt{\rho}}\Bigr\}c_lD^2+\frac{1}{c_l\sqrt{\rho}}$. We arrive at:
\begin{equation}
    \E[f(\Bar{\xx}_{T})-f^\star]\le\frac{\tau^2}{T}+\frac{\tau\sigma_{f,B}}{\sqrt{T}}.
\end{equation}
where $\Bar{\xx}_T=\frac{1}{T}\sum_{t=0}^{T-1}\xx_t$.
\end{proof}

\subsection{Full statement and proof for Lemma \ref{lemma:bounded iterates}}
\label{appendix:sec:lemma3}
\begin{lemma}[Bounded iterates]
    Let each $f_i$ be $\mu$-strongly convex and $L$-smooth. For any $t\in\mathbb{N}$, the iterates of SGD with AdaSPS or AdaSLS satisfy: 
    \begin{equation}
    \begin{split}
    ||\xx_t-\xx^\star||^2\le D_{\max}:=\max\left\{||\xx_0-\xx^\star||^2,\frac{2\sigma_{\max}^2+b} {\mu}, (2\sigma_{\max}^2+b)\eta_0\right\}\;,
    \end{split} \label{eq:dmax}
    \end{equation}
    where $\sigma_{\max}^2:=\max_{i_t}\left\{f_{i_t}(\xx^\star)-\ell^\star_{i_t}\right\}$, $b:=1/\Bigl(4c_p^3\sqrt{f_{i_0}(\xx_0)-\ell_{i_0}^\star}\Bigr)$ for AdaSPS and $\sigma_{\max}^2:=\max_{i_t}\left\{f_{i_t}(\xx^\star)-f^\star_{i_t}\right\}$, $b:=1/\Bigl(4c_l^3\rho^2\sqrt{\gamma_0||\nabla f_{i_0}(\xx_0)||^2}\Bigr)$ for AdaSLS.
\end{lemma}
\begin{proof}
By strong convexity of $f_{i_t}$, the iterates generated by SGD satisfy: 
\begin{equation}
\begin{split}
    ||\xx_{t+1}-\xx^\star||^2&=||\xx_t-\xx^\star||^2-2\eta_t\langle\nabla f_{i_t}(\xx_t),\xx_t-\xx^\star\rangle+\eta_t^2||\nabla f_{i_t}(\xx_t)||^2 \\ 
    &\overset{\eqref{df:stconvex}}{\le}||\xx_t-\xx^\star||^2-2\eta_t(f_{i_t}(\xx_t)-f_{i_t}(\xx^\star)+\frac{\mu}{2}||\xx_t-\xx^\star||^2)+\eta_t^2||\nabla f_{i_t}(\xx_t)||^2 \\ 
    &=(1-\eta_t\mu)||\xx_t-\xx^\star||^2-2\eta_t(f_{i_t}(\xx_t)-f_{i_t}(\xx^\star))+\eta_t^2||\nabla f_{i_t}(\xx_t)||^2\;.
    \label{eq:basic_ineq}
\end{split}
\end{equation}
We next separate the proofs for clarity.

\textbf{AdaSPS}: Plugging in the upper bound of $\eta_t$ in Lemma \ref{lemma:step_bound_adasps}, we obtain:
\begin{equation}
\begin{split}
&||\xx_{t+1}-\xx^\star||^2\\ 
&\stackrel{\eqref{eq:step_bound_adasps}}{\le}(1-\eta_t\mu)||\xx_t-\xx^\star||^2-2\eta_t(f_{i_t}(\xx_t)-f_{i_t}(\xx^\star))+\eta_t\frac{f_{i_t}(\xx_t)-\ell_{i_t}^\star}{c_p\sqrt{\sum_{s=0}^tf_{i_s}(\xx_t)-\ell_{i_s}^\star}} \\ 
&=(1-\eta_t\mu)||\xx_t-\xx^\star||^2+2\eta_t(f_{i_t}(\xx^\star)-\ell_{i_t}^\star)-2\eta_t(f_{i_t}(\xx_t)-\ell_{i_t}^\star)+\eta_t\frac{f_{i_t}(\xx_t)-\ell_{i_t}^\star}{c_p\sqrt{\sum_{s=0}^tf_{i_s}(\xx_t)-\ell_{i_s}^\star}} \\ 
&\le(1-\eta_t\mu)||\xx_t-\xx^\star||^2+2\eta_t\sigma_{\max}^2-2\eta_t\underbrace{(f_{i_t}(\xx_t)-\ell_{i_t}^\star)}_{\ge0}+\eta_t\frac{f_{i_t}(\xx_t)-\ell_{i_t}^\star}{c_p\sqrt{\sum_{s=0}^tf_{i_s}(\xx_t)-\ell_{i_s}^\star}}\;,
\end{split}
\end{equation}
where $\sigma_{\max}^2:=\max_{i_t}\left\{f_{i_t}(\xx^\star)-\ell^\star_{i_t}\right\}$.

We now split the proof into two cases. Firstly, if $c_p\sqrt{\sum_{s=0}^tf_{i_s}(\xx_t)-\ell_{i_s}^\star}\le\frac{1}{2}$, then it follows that:
\begin{equation}
    f_{i_t}(\xx_t)-\ell_{i_t}^\star\le(\frac{1}{2c_p})^2\quad\text{and}\quad \sum_{s=0}^tf_{i_s}(\xx_t)-\ell_{i_s}^\star\ge f_{i_0}(\xx_0)-\ell_{i_0}^\star\;.
\end{equation}
Plugging in the above bounds, we have:
\begin{equation}
\begin{split}
||\xx_{t+1}-\xx^\star||^2\le(1-\eta_t\mu)||\xx_t-\xx^\star||^2+\eta_t(2\sigma_{\max}^2+\frac{1}{4c_p^3\sqrt{f_{i_0}(\xx_0)-\ell_{i_0}^\star}})\;.
\end{split}
\end{equation} 
We conclude by applying Lemma \ref{lemma:decsps} with $z_t=||\xx_t-\xx^\star||^2$, $a=\mu$, $b=(2\sigma_{\max}^2+\frac{1}{4c_p^3\sqrt{f_{i_0}(\xx_0)-\ell_{i_0}^\star}})$.
Secondly, if instead $c_p\sqrt{\sum_{s=0}^tf_{i_s}(\xx_t)-\ell_{i_s}^\star}\ge\frac{1}{2}$, then we have:
\begin{equation}
    -2\eta_t(f_{i_t}(\xx_t)-\ell_{i_t}^\star)+\eta_t\frac{f_{i_t}(\xx_t)-\ell_{i_t}^\star}{c_p\sqrt{\sum_{s=0}^tf_{i_s}(\xx_t)-\ell_{i_s}^\star}}\le 0\;,
\end{equation}
and consequently we can apply Lemma \ref{lemma:decsps} with $z_t=||\xx_t-\xx^\star||^2$, $a=\mu$, $b=2\sigma_{\max}^2$.

\textbf{AdaSLS}: Similarly, by plugging the upper bound of $\eta_t$ in Lemma \ref{lemma:step_bound_adasls}, we obtain:
\begin{equation}
\begin{split}
&||\xx_{t+1}-\xx^\star||^2\\ 
&\stackrel{\eqref{eq:step_bound_adasls}}{\le}(1-\eta_t\mu)||\xx_t-\xx^\star||^2-2\eta_t(f_{i_t}(\xx_t)-f_{i_t}(\xx^\star))+\eta_t\frac{\gamma_t||\nabla f_{i_t}(\xx_t)||^2}{c_l\sqrt{\sum_{s=0}^t\gamma_s\||\nabla f_{i_s}(\xx_s)||^2}} \\ 
&\le(1-\eta_t\mu)||\xx_t-\xx^\star||^2+2\eta_t(f_{i_t}(\xx^\star)-f_{i_t}^\star)-2\eta_t(f_{i_t}(\xx_t)-f_{i_t}^\star)+\eta_t\frac{f_{i_t}(\xx_t)-f_{i_t}^\star}{c_l\rho\sqrt{\sum_{s=0}^t\gamma_s\||\nabla f_{i_s}(\xx_s)||^2}} \\ 
&\le(1-\eta_t\mu)||\xx_t-\xx^\star||^2+2\eta_t\sigma_{\max}^2-2\eta_t\underbrace{(f_{i_t}(\xx_t)-f_{i_t}^\star)}_{\ge0}+\eta_t\frac{f_{i_t}(\xx_t)-f_{i_t}^\star}{c_l\rho\sqrt{\sum_{s=0}^t\gamma_s\||\nabla f_{i_s}(\xx_s)||^2}}\;,
\end{split}
\end{equation}
where $\sigma_{\max}^2=\max_{i_t}\left\{f_{i_t}(\xx^\star)-f^\star_{i_t}\right\}$. We can then compare $c_l\rho\sqrt{\sum_{s=0}^t\gamma_s\||\nabla f_{i_s}(\xx_s)||^2}$ with $\frac{1}{2}$ and apply Lemma \ref{lemma:decsps} correspondingly. 
\end{proof}

\subsection{Proof for Theorem \ref{theorem:convex+interp} and \ref{thm:stconvex}}
\label{appendix:sec:interp}

\begin{proof}
    For clarity, we separate the proofs for AdaSPS and AdaSLS.

    \textbf{AdaSPS:} Plugging in the upper bound of $\eta_t$ in Lemma \ref{lemma:step_bound_adasps}, we have:
    \begin{equation}
    \begin{split}
    ||\xx_{t+1}-\xx^\star||^2 
    &\stackrel{\eqref{eq:step_bound_adasps}}{\le}||\xx_t-\xx^\star||^2-2\eta_t\langle \nabla f_{i_t}(\xx_t), \xx_t-\xx^\star\rangle+\eta_t\frac{f_{i_t}(\xx_t)-f^\star}{c_p\sqrt{\sum_{s=0}^{t}f_{i_s}(\xx_s)-f^\star}}\;.
    \end{split}
    \label{eq:thm:stconvex:1}
    \end{equation} 
    Since $c_p\sqrt{f_{i_0}(\xx_0)-f^\star}\ge1$, \eqref{eq:thm:stconvex:1} can be reduced to:
\begin{equation}
    \begin{split}
    ||\xx_{t+1}-\xx^\star||^2 
    &\le||\xx_t-\xx^\star||^2-2\eta_t\langle \nabla f_{i_t}(\xx_t), \xx_t-\xx^\star\rangle+\eta_t(f_{i_t}(\xx_t)-f^\star)\;.
    \end{split}
    \end{equation} 
By convexity of $f_{i_t}$, we get:
\begin{equation}
    \begin{split}
    ||\xx_{t+1}-\xx^\star||^2 
    &\stackrel{\eqref{df:convex}}{\le}||\xx_t-\xx^\star||^2-\eta_t\langle \nabla f_{i_t}(\xx_t), \xx_t-\xx^\star\rangle\;.
    \label{eq:thm:stconvex:2}
    \end{split}
\end{equation} 
Note that $\langle \nabla f_{i_t}(\xx_t), \xx_t-\xx^\star\rangle\ge0$ and $\eta_t$ is non-increasing, we thus get:
\begin{equation}
    \begin{split}
    ||\xx_{t+1}-\xx^\star||^2 
    \le||\xx_t-\xx^\star||^2-\eta_{T-1}\langle \nabla f_{i_t}(\xx_t), \xx_t-\xx^\star\rangle\;.
    \label{eq:thm:stconvex:3}
    \end{split}
\end{equation} 

We first show that $\eta_{T-1}$ is always lower bounded. From equation \eqref{eq:thm:stconvex:2} and using convexity of $f_{i_t}$, we get:
\begin{equation}
    \eta_t(f_{i_t}(\xx_t)-f^\star)\le||\xx_t-\xx^\star||^2-||\xx_{t+1}-\xx^\star||^2\;.
\end{equation}
Summing from $t=0$ to $t=T-1$, we get:
\begin{equation}
    \sum_{t=0}^{T-1} \eta_t(f_{i_t}(\xx_t)-f^\star)\le||\xx_0-\xx^\star||^2\;.
\end{equation}
Using the lower bound of $\eta_t$, we get:
\begin{equation}
    \frac{1}{2c_pL}\sqrt{\sum_{s=0}^{T-1}f_{i_s}(\xx_s)-f^\star}\stackrel{\eqref{eq:wald}}{\le}\frac{1}{2c_pL}\sum_{t=0}^{T-1}\frac{f_{i_t}(\xx_t)-f^\star}{\sqrt{\sum_{s=0}^tf_{i_s}(\xx_s)-f^\star}}\stackrel{\eqref{eq:step_bound_adasps}}{\le}\sum_{t=0}^{T-1} \eta_t(f_{i_t}(\xx_t)-f^\star)\;,
\end{equation}
This reveals that:
\begin{equation}
    \eta_{T-1}\stackrel{\eqref{eq:step_bound_adasps}}{\ge}\frac{1}{2c_pL}\frac{1}{\sqrt{\sum_{s=0}^{T-1}f_{i_s}(\xx_s)-f^\star}}\ge\frac{1}{(2c_pL||\xx_0-\xx^\star||)^2}\;.
\end{equation}
Plugging in the lower bound of $\eta_{T-1}$ to \eqref{eq:thm:stconvex:3}, we obtain:
\begin{equation}
    \begin{split}
    ||\xx_{t+1}-\xx^\star||^2 
    \le||\xx_t-\xx^\star||^2-\frac{1}{(2c_pL||\xx_0-\xx^\star||)^2}\langle \nabla f_{i_t}(\xx_t), \xx_t-\xx^\star\rangle\;.
    \end{split}
    \label{eq:thm:stconvex:4}
\end{equation} 

Plugging in $c_p=\frac{c_p^{\text{scale}}}{\sqrt{f_{i_0}(\xx_0)-f^\star}}$, we get
\begin{equation}
    \begin{split}
    ||\xx_{t+1}-\xx^\star||^2 
    \le||\xx_t-\xx^\star||^2-\frac{f_{i_0}(\xx_0)-f^\star}{(2c_p^{\text{scale}}L||\xx_0-\xx^\star||)^2}\langle \nabla f_{i_t}(\xx_t), \xx_t-\xx^\star\rangle\;.
     \label{eq:thm:stconvex:5}
    \end{split}
\end{equation} 

\textbf{case 1: $f$ is convex}.

For any $t\ge 1$, we take expectation conditional on $i_0$ on both sides and get:
\begin{equation}
    \begin{split}
    \E[||\xx_{t+1}-\xx^\star||^2|i_0] 
    &\le\E[||\xx_t-\xx^\star||^2|i_0]-\frac{f_{i_0}(\xx_0)-f^\star}{(2c_p^{\text{scale}}L||\xx_0-\xx^\star||)^2}\E[\langle \nabla f_{i_t}(\xx_t), \xx_t-\xx^\star\rangle|i_0] \\
    &\stackrel{\eqref{df:convex}}{\le}\E[||\xx_t-\xx^\star||^2|i_0]-\frac{f_{i_0}(\xx_0)-f^\star}{(2c_p^{\text{scale}}L||\xx_0-\xx^\star||)^2}\E[f(\xx_t)-f^\star|i_0]  \;.
    \end{split}
\end{equation} 
Summing up from $t=1$ to $t=T$ and dividing by $T$, we obtain:
\begin{equation}
    \frac{1}{T}\sum_{t=1}^T\E[f(\xx_t)-f^\star|i_0]\le 4L(c_p^{\text{scale}})^2\frac{||\xx_0-\xx^\star||^2}{f_{i_0}(\xx_0)-f^\star}\frac{L\E[||\xx_1-\xx^\star||^2|i_0]}{T}\;.
\end{equation}
Note that $\E[||\xx_1-\xx^\star||^2|i_0]=||\xx_1-\xx^\star||^2\le||\xx_0-\xx^\star||^2$ due to \eqref{eq:thm:stconvex:4}. We thus get:
\begin{equation}
    \frac{1}{T}\sum_{t=1}^T\E[f(\xx_t)-f^\star|i_0]\le \Bigl(4L(c_p^{\text{scale}})^2\frac{||\xx_0-\xx^\star||^2}{f_{i_0}(\xx_0)-f^\star}\Bigr)\frac{L||\xx_0-\xx^\star||^2}{T}\;
\end{equation}

Taking expectation w.r.t $i_0$ on both sides, we arrive at:
\begin{equation}
    \frac{1}{T}\sum_{t=1}^T\E[f(\xx_t)-f^\star]\le A\frac{L||\xx_0-\xx^\star||^2}{T}\;\text{with}\;A:=4L(c_p^{\text{scale}})^2\E_{i_0}[\frac{||\xx_0-\xx^\star||^2}{f_{i_0}(\xx_0)-f^\star}]\;.
    \label{def:A:AdaSPS}
\end{equation}
\textbf{case 2: $f$ is strongly-convex}

For any $t\ge 1$, we take expectation conditional on $i_0$ on both sides of 
\eqref{eq:thm:stconvex:5} and get:
\begin{equation}
    \begin{split}
    \E[||\xx_{t+1}-\xx^\star||^2|i_0] 
    &\le\E[||\xx_t-\xx^\star||^2|i_0]-\frac{f_{i_0}(\xx_0)-f^\star}{(2c_p^{\text{scale}}L||\xx_0-\xx^\star||)^2}\E[\langle \nabla f_{i_t}(\xx_t), \xx_t-\xx^\star\rangle|i_0] \\
    &\stackrel{\eqref{df:stconvex}}{\le}\E[||\xx_t-\xx^\star||^2|i_0]-\frac{f_{i_0}(\xx_0)-f^\star}{(2c_p^{\text{scale}}L||\xx_0-\xx^\star||)^2}\E[f(\xx_t)-f^\star+\frac{\mu}{2}||\xx_t-\xx^\star||^2|i_0]\;.
    \end{split}
\end{equation}
Note that $f(\xx_t)-f^\star\ge\frac{\mu}{2}||\xx_t-\xx^\star||^2$ due to strong convexity, we obtain:
\begin{equation}
    \E[||\xx_{t+1}-\xx^\star||^2|i_0] 
    \le\Bigl(1-\frac{(f_{i_0}(\xx_0)-f^\star)\mu}{(2c_p^{\text{scale}}L||\xx_0-\xx^\star||)^2}\Bigr) \E[||\xx_{t}-\xx^\star||^2|i_0] \;.
\end{equation}
For any $T\ge 1$, unrolling gives:
\begin{equation}
\begin{split}
     \E[||\xx_{T+1}-\xx^\star||^2|i_0] 
    &\le\Bigl(1-\frac{(f_{i_0}(\xx_0)-f^\star)\mu}{(2c_p^{\text{scale}}L||\xx_0-\xx^\star||)^2}\Bigr)^T\E[||\xx_{1}-\xx^\star||^2|i_0] \\
    &\le \Bigl(1-\frac{(f_{i_0}(\xx_0)-f^\star)\mu}{(2c_p^{\text{scale}}L||\xx_0-\xx^\star||)^2}\Bigr)^T||\xx_{0}-\xx^\star||^2\;.
\end{split}
\end{equation}
The claim of Theorem~\ref{thm:stconvex} follows by taking expectation w.r.t $i_0$ on both sides.

\textbf{AdaSLS}: We only highlight the difference from AdaSPS. Plugging in the upper bound of $\eta_t$ from Lemma \ref{lemma:step_bound_adasls} and using the line-search condition, we get:
\begin{equation}
    \begin{split}
    ||\xx_{t+1}-\xx^\star||^2 
    &\stackrel{\eqref{eq:step_bound_adasls}}{\le}||\xx_t-\xx^\star||^2-2\eta_t\langle \nabla f_{i_t}(\xx_t), \xx_t-\xx^\star\rangle+\eta_t\frac{\gamma_t||\nabla f_{i_t}(\xx_t)||^2}{c_l\sqrt{\sum_{s=0}^{t}\gamma_s||\nabla f_{i_s}(\xx_s)||^2}} \\ 
    &\le||\xx_t-\xx^\star||^2-2\eta_t\langle \nabla f_{i_t}(\xx_t), \xx_t-\xx^\star\rangle+\eta_t\frac{f_{i_t}(\xx_t)-f^\star}{c_l\rho\sqrt{\sum_{s=0}^{t}\gamma_s||\nabla f_{i_s}(\xx_s)||^2}}\;. 
    \end{split}
    \end{equation} 
   
Since $c_l\rho\sqrt{\sum_{s=0}^{t}\gamma_s||\nabla f_{i_s}(\xx_s)||^2}\ge 1$, we obtain the same equation as \eqref{eq:thm:stconvex:3}. To find a lower bound of $\eta_{T-1}$, we rearrange \eqref{eq:thm:stconvex:2} as:
\begin{equation}
    \eta_t(f_{i_t}(\xx_t)-f^\star)\le||\xx_t-\xx^\star||^2-||\xx_{t+1}-\xx^\star||^2\;,
\end{equation}
the left-hand-side of which can be lower bounded by:
\begin{equation}
\begin{split}
    \eta_t(f_{i_t}(\xx_t)-f^\star)
    \ge\eta_t\rho\gamma_t||\nabla f_{i_t}(\xx_t)||^2\stackrel{\eqref{eq:step_bound_adasls}}{\ge}\min\{\frac{1-\rho}{L},\gamma_{\max}\}\frac{\rho\gamma_t||\nabla f_{i_t}(\xx_t)||^2}{c_l\sqrt{\sum_{s=0}^t\gamma_s||\nabla f_{i_s}(\xx_s)||^2}}\;.
\end{split}
\end{equation}
Summing over $t=0$ to $t=T-1$ gives:
\begin{equation}
    \min\{\frac{1-\rho}{L},\gamma_{\max}\}\frac{\rho}{c_l}\sqrt{\sum_{s=0}^{T-1}\gamma_s||\nabla f_{i_s}(\xx_s)||^2}\stackrel{\eqref{eq:wald}}{\le}\min\{\frac{1-\rho}{L},\gamma_{\max}\}\frac{\rho}{c_l}\sum_{t=0}^{T-1}\frac{\gamma_t||\nabla f_{i_t}(\xx_t)||^2}{\sqrt{\sum_{s=0}^t\gamma_s||\nabla f_{i_s}(\xx_s)||^2}}\le||\xx_0-\xx^\star||^2\;.
    \label{eq:thm:stconvex:-1}
\end{equation}
This implies that:
\begin{equation}
    \eta_{T-1}\stackrel{\eqref{eq:step_bound_adasls}}{\ge}\frac{\min\{\frac{1-\rho}{L},\gamma_{\max}\}}{c_l\sqrt{\sum_{s=0}^{T-1}\gamma_s||\nabla f_{i_s}(\xx_s)||^2}}\stackrel{\eqref{eq:thm:stconvex:-1}}{\ge}\frac{\rho\min^2\{\frac{1-\rho}{L},\gamma_{\max}\}}{c_l^2||\xx_0-\xx^\star||^2}\;.
\end{equation}
After plugging the above into \eqref{eq:thm:stconvex:3}, the remaining proof follows from the same routine as shown for AdaSPS.
\end{proof}

\paragraph{Proofs for Loopless Variance-Reduction}
\subsection{Proof of Lemma \ref{lemma:vr}}
\begin{proof}
    By definition of $F_{i_t}(\xx)$, we have:
    \begin{equation}
        \begin{split}
        &F_{i_t}(\xx_t)-F_{i_t}^\star\\
        &=F_{i_t}(\xx_t)-F_{i_t}(\xx^\star)+F_{i_t}(\xx^\star)-F_{i_t}^\star \\ 
        &= f_{i_t}(\xx_t)-f_{i_t}(\xx^\star)+(\xx_t-\xx^\star)^T(\nabla f(\ww_t)-\nabla f_{i_t}(\ww_t))-\frac{\mu_F}{2}||\xx_t-\xx^\star||^2+F_{i_t}(\xx^\star)-F_{i_t}^\star\;.
        \end{split}
        \label{eq:lemma:vr:1}
    \end{equation}
    By $\mu_F$-strong convexity of $F_{i_t}(\xx)$, we obtain:
    \begin{equation}
        \begin{split}
            F_{i_t}(\xx^\star)-F_{i_t}^\star&\stackrel{\eqref{eq:step_bound}}{\le}\frac{1}{2\mu_F}||\nabla F_{i_t}(\xx^\star)||^2 \\ 
            &=\frac{1}{2\mu_F}||\nabla f_{i_t}(\xx^\star)-\nabla f_{i_t}(\ww_t)+\nabla f(\ww_t)+\mu_F(\xx^\star-\xx_t)||^2\;.
        \end{split}
        \label{eq:lemma:vr:2}
    \end{equation}
    Plugging \eqref{eq:lemma:vr:2} into \eqref{eq:lemma:vr:1}, taking expectation w.r.t the randomness $i_t$ on both sides gives:
    \begin{equation}
    \begin{split}
        &\E_{i_t}[F_{i_t}(\xx_t)-F_{i_t}^\star]\\&\le
        f(\xx_t)-f(\xx^\star)-\frac{\mu_F}{2}||\xx_t-\xx^\star||^2+\E_{i_t}[\frac{1}{2\mu_F}||\nabla f_{i_t}(\xx^\star)-\nabla f_{i_t}(\ww_t)+\nabla f(\ww_t)+\mu_F(\xx^\star-\xx_t)||^2] \\ 
        &=f(\xx_t)-f(\xx^\star)-\frac{\mu_F}{2}||\xx_t-\xx^\star||^2+\frac{\mu_F}{2}||\xx_t-\xx^\star||^2+\frac{1}{2\mu_F}\E_{i_t}[||\nabla f_{i_t}(\xx^\star)-\nabla f_{i_t}(\ww_t)+\nabla f(\ww_t)||^2]\\ 
        &\quad+\frac{1}{\mu_F}\E_{i_t}[\langle\nabla f_{i_t}(\xx^\star)-\nabla f_{i_t}(\ww_t)+\nabla f(\ww_t),\mu_F(\xx^\star-\xx_t)\rangle] \\ 
        &=f(\xx_t)-f(\xx^\star)+\frac{1}{2\mu_F}\E_{i_t}[||\nabla f_{i_t}(\xx^\star)-\nabla f_{i_t}(\ww_t)+\nabla f(\ww_t)||^2] \\ 
        &\le f(\xx_t)-f(\xx^\star)+\frac{1}{2\mu_F}\E_{i_t}[||\nabla f_{i_t}(\xx^\star)-\nabla f_{i_t}(\ww_t)||^2]\;.
    \end{split}
    \label{eq:lemma7:upbound}
    \end{equation}
\end{proof}

\subsection{Proof of Theorem \ref{thm:vr}}
\begin{proof}
    Recall the proof technique that gives equation \eqref{thm:vr:ref:1} and \eqref{eq:thm:vr:ref:2} in Theorem \ref{thm:convex}. Following the same routine, we arrive at:
    \begin{equation}
    \sum_{t=0}^{T-1}\langle\nabla F_{i_t}(\xx_t),\xx_t-\xx^\star\rangle \le \tau\sqrt{\sum_{t=0}^{T-1} F_{i_t}(\xx_t)-F_{i_t}^\star}\;.
    \end{equation}
where $\tau=(2c_p(L+\mu_F)D^2+\frac{1}{c_p})$ for AdaSVRPS and $\tau=\max\Bigl\{\frac{L+\mu_F}{(1-\rho)\sqrt{\rho}},\frac{1}{\gamma_{\max}\sqrt{\rho}}\Bigr\}c_lD^2+\frac{1}{c_l\sqrt{\rho}}$ for AdaSVRLS. The difference is due to the fact that $F_{i_t}(\xx)$ is $(L+\mu_F)$-smooth. Taking the expectation, using the fact that $\E[\nabla F_{i_t}(\xx_t)]=\E[\nabla f_{i_t}(\xx_t)+\nabla f(\ww_t)-\nabla f_{i_t}(\ww_t)]=\nabla f(\xx_t)$ and applying Lemma \ref{lemma:vr}, we end up with:
\begin{equation}
    \begin{split}
        \sum_{t=0}^{T-1} \E[f(\xx_t)-f^\star]\stackrel{\eqref{eq:lemma7:upbound}}{\le}\tau \sqrt{\sum_{t=0}^{T-1}\E[f(\xx_t)-f^\star]+\frac{1}{2\mu_F}\sum_{t=0}^{T-1}\E[||\nabla f_{i_t}(\ww_t)-\nabla f_{i_t}(\xx^\star)||^2]}\;.
    \end{split}
\end{equation} 
Taking the square gives:
\begin{equation}
    \begin{split}
        \Biggl(\sum_{t=0}^{T-1} \E[f(\xx_t)-f^\star]\Biggr)^2\le\tau^2\Biggl(\sum_{t=0}^{T-1}\E[f(\xx_t)-f^\star]+\frac{1}{2\mu_F}\sum_{t=0}^{T-1}\E[||\nabla f_{i_t}(\ww_t)-\nabla f_{i_t}(\xx^\star)||^2]\Biggr)\;.
    \end{split}
    \label{eq:thm:vr:square}
\end{equation} 
Define the Lyapunov function: $\cZ_{t+1}=\frac{1}{2(1-a)}\frac{\tau^2}{p_{t+1}\mu_F}||\nabla f_{i_{t+1}}(\ww_{t+1})-\nabla f_{i_{t+1}}(\xx^\star)||^2$. It follows that:
\begin{equation}
    \begin{split}
        \E[\cZ_{t+1}]&=\frac{1}{2(1-a)}\frac{\tau^2}{p_{t+1}\mu_F}\E[||\nabla f_{i_{t+1}}(\ww_{t+1})-\nabla f_{i_{t+1}}(\xx^\star)||^2] \\ 
        &=\frac{\tau^2}{2(1-a)\mu_F}\E[||\nabla f_{i_{t+1}}(\xx_{t})-\nabla f_{i_{t+1}}(\xx^\star)||^2]+\frac{1-p_{t+1}}{2(1-a)}\frac{\tau^2}{p_{t+1}\mu_F}\E[||\nabla f_{i_{t}}(\ww_{t})-\nabla f_{i_{t}}(\xx^\star)||^2]\\ 
        &\stackrel{\eqref{eq:smooth+convex:property}}{\le}\frac{\tau^2}{2(1-a)\mu_F}\E[2L(f_{i_{t+1}}(\xx_t)-f_{i_{t+1}}(\xx^\star)-\langle \nabla f_{i_{t+1}}(\xx^\star),\xx_t-\xx^\star\rangle)]+\frac{(1-p_{t+1})p_t}{p_{t+1}}\E[\cZ_{t}] \\ 
        &=\frac{L}{(1-a)\mu_F}\tau^2(\E[f(\xx_t)-f^\star])+\frac{(1-p_{t+1})p_t}{p_{t+1}}\E[\cZ_{t}]\;.
    \end{split}
\end{equation}
Adding $\sum_{t=0}^{T-1}\E[\cZ_{t+1}]$ to both sides of \eqref{eq:thm:vr:square} and substituting the above upper bound, we get:
\begin{equation}
    \begin{split}
        \Biggl(\sum_{t=0}^{T-1} \E[f(\xx_t)-f^\star]\Biggr)^2+\sum_{t=0}^{T-1}\E[\cZ_{t+1}]&\le(1+\frac{L}{(1-a)\mu_F})\tau^2\sum_{t=0}^{T-1}\E[f(\xx_t)-f^\star]\\
        &\quad+\sum_{t=0}^{T-1}\Bigl((1-a)p_t+\frac{(1-p_{t+1})p_t}{p_{t+1}}\Bigr)\E[\cZ_{t}]\;.
    \end{split}
\end{equation}
Rearranging and dropping $\E[\cZ_T]\ge0$ gives:
\begin{equation}
    \begin{split}
        \Biggl(\sum_{t=0}^{T-1} \E[f(\xx_t)-f^\star]\Biggr)^2&\le(1+\frac{L}{(1-a)\mu_F})\tau^2\sum_{t=0}^{T-1}\E[f(\xx_t)-f^\star]+\sum_{t=1}^{T-1}\Bigl((1-a)p_t+\frac{(1-p_{t+1})p_t}{p_{t+1}}-1\Bigr)\E[\cZ_{t}]\\ 
        &\quad+\bigl((1-a)p_0+\frac{(1-p_1)p_0}{p_1}\bigr)\E[\cZ_0]\;.
    \end{split}
\end{equation}
By our choice of $p_t$, we have: 
\begin{equation}
    (1-a)p_t+\frac{(1-p_{t+1})p_t}{p_{t+1}}-1=\frac{p_t}{p_{t+1}}-ap_t-1=\frac{at+a+1}{at+1}-\frac{a}{(at+1)}-1=\frac{0}{2(at+1)}= 0\;.
\end{equation}
Therefore, it holds that:
\begin{equation}
    \begin{split}
        \Biggl(\sum_{t=0}^{T-1} \E[f(\xx_t)-f^\star]\Biggr)^2&\le(1+\frac{L}{(1-a)\mu_F})\tau^2\sum_{t=0}^{T-1}\E[f(\xx_t)-f^\star]+\E[\cZ_0]\;.
    \end{split}
\end{equation}
Further, by $L$-smoothness and convexity of $f$, we have
\begin{equation}
    \E[\cZ_0]=\frac{1}{2(1-a)}\frac{\tau^2}{p_0\mu_F}\E[||\nabla f_{i_0}(\xx_0)-\nabla f_{i_0}(\xx^\star)||^2]\stackrel{\eqref{eq:smooth+convex:property}}{\le}\frac{L\tau^2}{(1-a)\mu_F}(f(\xx_0)-f^\star)\;.
\end{equation}
Hence. we obtain:
\begin{equation}
    \Biggl(\sum_{t=0}^{T-1} \E[f(\xx_t)-f^\star]\Biggr)^2\le(1+\frac{2L}{(1-a)\mu_F})\tau^2\sum_{t=0}^{T-1}\E[f(\xx_t)-f^\star].
\end{equation}
It follows that:
\begin{equation}
    \sum_{t=0}^{T-1} \E[f(\xx_t)-f^\star]\le(1+\frac{2L}{(1-a)\mu_F})\tau^2.
\end{equation}

Dividing both sides by $T$ and applying Jensen's inequality concludes the proof.

\end{proof}

\section{Pseudo-codes for AdaSPS and AdaSLS}
In this section, we provide formal pseudo codes for AdaSPS \eqref{AdaSPS} and AdaSLS \eqref{AdaSLS}.

To implement AdaSPS, a lower bound of optimal function value for each minibatch function is required. For machine learning problems where the individual loss functions are non-negative, we can use zero as an input. Apart from that, we need to provide a constant $c_p$ to adjust the magnitude of the stepsize. Theoretically suggested $c_p$ for robust convergence satisfies $c^{\text{scale}}_p=c_p\sqrt{f_{i_0}(\xx_0)-\ell_{i_0}^\star}\ge\frac{1}{2}$. Therefore, a common choice is to set $c_p=\frac{1}{\sqrt{f_{i_0}(\xx_0)-\ell_{i_0}^\star}}$.

\begin{algorithm}[H]
\begin{algorithmic}[1]
\caption{AdaSPS}
\Require $\xx_0\in\R^d$, 
$T\in\mathbb{N}^+$, $c_p>0$
\State set $\eta_{-1}=+\infty$
\State set $\epsilon=10^{-10}$
\For{$t=0$ to $T-1$}
 \State uniformly sample $i_t\subseteq[n]$
 \State provide a lower bound $\ell_{i_t}^\star\le f_{i_t}^\star$
 \State set $\eta_t=\min\left\{\frac{f_{i_t}(\xx_t)-\ell_{i_t}^\star}{c_p||\nabla f_{i_t}(\xx_t)||^2}\frac{1}{\sqrt{\sum_{s=0}^t f_{i_s}(\xx_s)-\ell_{i_s}^\star}+\epsilon},\eta_{t-1}\right\}$
 \State $\xx_{t+1}=\Pi_{\cX}(\xx_t-\eta_t\nabla f_{i_t}(\xx_t))$
\EndFor
\Return $\Bar{\xx}_T=\frac{1}{T}\sum_{t=0}^{T-1}\xx_t$
\end{algorithmic}
\end{algorithm}

To implement AdaSLS \eqref{AdaSLS}, a line-search sub-solver~\ref{Algorithm:Armijo} and an input constant $c_l>0$ are required. Similar to \eqref{AdaSPS}, we can set $c_l=\frac{1}{\rho\sqrt{\gamma_0||\nabla f_{i_0}(\xx_0)||^2}}$ according to the theory.
\begin{algorithm}[H]
\begin{algorithmic}[1]
\caption{AdaSLS}
\Require $\xx_0\in\R^d$, 
$T\in\mathbb{N}^+$, $c_l>0$
\State set $\eta_{-1}=+\infty$
\State set $\epsilon=10^{-10}$
\For{$t=0$ to $T-1$}
 \State uniformly sample $i_t\subseteq[n]$
 \State obtain $\gamma_t$ via backtracking line-search (\ref{Algorithm:Armijo})  
 \State set $\eta_t=\min\left\{\frac{\gamma_t}{c_l\sqrt{\sum_{s=0}^t \gamma_{s}||\nabla f_{i_s}(\xx_s)||^2}+\epsilon},\eta_{t-1}\right\}$
 \State $\xx_{t+1}=\Pi_{\cX}(\xx_t-\eta_t\nabla f_{i_t}(\xx_t))$
\EndFor
\Return $\Bar{\xx}_T=\frac{1}{T}\sum_{t=0}^{T-1}\xx_t$
\end{algorithmic}
\end{algorithm}

\section{Line-search procedure}
\label{sec:line-search}
In this section, we introduce the classical Armijo line-search method~\cite{armijo-line-search,num_opt_book}. Given a function $f_{i_t}(\xx)$, the Armijo line-search returns a stepsize $\gamma_t$ that satisfies the following condition:
\begin{equation}
    f_{i_t}(\xx_t-\gamma_t\nabla f_{i_t}(
\xx_t))\le f_{i_t}(\xx_t)-\rho\gamma_t||\nabla f_{i_t}(\xx_t)||^2\;,
\label{eq:armijo_condition}
\end{equation}
where $\rho\in(0,1)$ is an input hyper-parameter. If $f_{i_t}(\xx)$ is a smooth function, then backtracking line-search~\ref{Algorithm:Armijo} is a practical  implementation way to ensure that \ref{eq:armijo_condition} is satisfied. 

\begin{algorithm}[h]
\caption{Backtracking line-search}
\label{Algorithm:Armijo}
\begin{algorithmic}[1]
\Require$\beta\in[\frac{1}{2},1)$, $\rho\in(0,1)$, $\gamma_{\max}>0$ (We fix $\beta=0.8$ and $\rho=0.5$ for AdaSLS)
\State $\gamma=\gamma_{\max}$
\While{$f_{i_t}(\xx_t-\gamma\nabla f_{i_t}(\xx_t))> f_{i_t}(\xx_t)-\rho\gamma||\nabla f_{i_t}(\xx_t)||^2$}
\State $\gamma=\beta\gamma$
\EndWhile
\Return $\gamma_t=\gamma$ 
\end{algorithmic}
\end{algorithm}

To implement Algorithm~\ref{Algorithm:Armijo}, one needs to provide the decreasing factor $\beta$, the maximum stepsize $\gamma_{\max}$, and the condition parameter $\rho$. Starting from $\gamma_{\max}$, Algorithm~\ref{Algorithm:Armijo} decreases the stepsize iteratively by a constant factor $\beta$ until the condition \ref{eq:armijo_condition} is satisfied. Note that checking the condition requires additional minibatch function value evaluations. Fortunately, note that the output $\gamma$ cannot be smaller than $\frac{1-\rho}{L}$ (Lemma \ref{lemma:step_bound_adasls}), and thus the number of extra function value evaluations required is at most $\cO\Bigl(\max\{\log_{1/\beta}^{L\gamma_{\max}/(1-\rho)},1\}\Bigr)$. In practice, \citet{sls} suggests dynamic initialization of $\gamma_{\max}$ to reduce the algorithm's running time, that is, setting $\gamma_{\max_t}=\gamma_{t-1}\theta^{1/n}$ where a common choice for $\theta$ is $2$. This strategy initializes $\gamma_{\max}$ by a slightly larger number than the last output and thus is usually more efficient than keeping $\gamma_{\max}$ constant or always using $\gamma_{\max_t}=\gamma_{t-1}$. In all our experiments, we use the same $\gamma_{\max}$ at each iteration for AdaSLS to show its theoretical properties.

Goldstein line-search is another line-search method that checks \ref{eq:armijo_condition} and an additional curvature condition \cite{num_opt_book}. We do not study this method in this work and we refer to \cite{sls} for more details. 

\section{Counter examples of SPS and its variants for SVRG}
\label{appendix:sec:counter_ex}
We provide two simple counterexamples where SVRG with the SPS stepsize and its intuitive variants fail to converge. For simplicity, consider the update rule $\xx_{t+1}=\xx_t-\eta_t\nabla f(\xx_t)$, i.e. $\ww_t=\xx_t$ for all $t\ge0$. Consider the function $f(\xx)=\frac{f_1(\xx)+f_2(\xx)}{2}$ where $f_1(\xx)=a_1(\xx-1)^2$ and $f_2(\xx)=a_2(\xx+1)^2$ with $a_1,a_2>0$.
\begin{example}
\textbf{Individual curvature is not representative.}
Consider the standard stochastic Polyak stepsize: $\eta_t=\frac{f_i(\xx_t)-f_i^\star}{||\nabla f_i(\xx_t)||^2}$ where $i$ is randomly chosen from $\{1,2\}$. We now let $a_1=1$ and $a_2<1$. Note that $\nabla^2f(\xx)=a_1+a_2\in(1,2)$
while $\E_i[\eta_t]=\frac{1}{8}+\frac{1}{8a_2}\to+\infty$ as $a_2\to 0$, which leads to divergence. The reason behind this is that individual curvature does not match the global curvature.
\end{example}
\begin{example}
    \textbf{Mismatching quantity.} Consider a variant of stochastic Polyak stepsize: $\eta_t=\frac{f_i(\xx_t)-f_i^\star}{||\nabla f_i(\xx_t)-\nabla f_i(\ww_t)+\nabla f(\ww_t)||^2}$ where $i$ is randomly chosen from $\{1,2\}$. Let $a_1=a_2=1$. We note $\E_i[\eta_t \nabla f(\xx_t)]=\frac{x_t^2+1}{2x_t}\neq 0$ and thus no stationary point exists. Similar reasoning can exclude a number of other variants such as: $\eta_t=\frac{f_i(\xx_t)-f_i(\ww_t)+f(\ww_t)-f_i^\star}{||\nabla f_i(\xx_t)-\nabla f_i(\ww_t)+\nabla f(\ww_t)||^2}$. Indeed, the numerator is not the proper function value difference of a valid function with the gradient defined in the denominator.
\end{example}

\section{Experimental details and additional experiment results}
\label{appendix:ex}
In this section, we provide a detailed setup of the experiments presented in the main paper. 

In practice, we can use a lower bound of $F_{i_t}^\star$ for running AdaSVRPS since convergence is still guaranteed thanks to the property of AdaSPS. By default, we use $\ell_{i_t}^\star+\min_\xx\{\xx^T(\nabla f(\ww_t)-\nabla f_{i_t}(\ww_t))+\frac{\mu_F}{2}||\xx-\xx_t||^2\}$ for all the experiments, where $\ell_{i_t}^\star$ is a lower bound for $f_{i_t}^\star$.

\subsection{Synthetic experiment}
\label{appendix:ex:syn}
We consider the minimization of a quadratic of the form: $f(\xx)=\frac{1}{n}\sum_{i=1}^nf_i(\xx)$ where $f_i(\xx)=\frac{1}{2}(\xx-\bb_i)^TA_i(\xx-\bb_i)$, $\bb_i\in\R^d$ and $A_i\in\R^{d\times d}$ is a diagonal matrix. We use $n=50$, $d=1000$. We control the interpolation by either setting all $\{b_i\}$ to be identical or different. Each component of $\{b_i\}$ is generated according to $\cN(0,10^2)$. We control the complexity of the problems by choosing different matrices $A_i$. 
For the strongly-convex case, we first generate a matrix $A^N=\text{clip}(\begin{pmatrix}a_{11}&...&a_{1d}\\...\\a_{n1}&...&a_{nd}\end{pmatrix},1,10)$ where each $a_{ij}\sim N(0,15^2)$ and the clipping operator clips the elements to the interval between 1 and 10. Then we compute:
\[
A=\begin{pmatrix}
    1&1&...&\frac{n}{\sum_{i=1}^n A^N_{i(d-1)}}&\frac{10n}{\sum_{i=1}^nA^N_{id}}\\
    ...& &...& &... \\
    1&1&...&\frac{n}{\sum_{i=1}^n A^N_{i(d-1)}}&\frac{10n}{\sum_{i=1}^nA^N_{id}}
\end{pmatrix} \bigodot A_N\;,
\]
where $\bigodot$ denotes the Hadamard product.  We set the diagonal elements of each $A_i$ using the corresponding row stored in the matrix $A$. Note that $\nabla^2 f(\xx)=\frac{1}{n}\sum_{i=1}^nA_i$ has the minimum and the largest eigenvalues being  1 and 10. For the general convex case,  we use the same matrix $A_N$ to generate a sparse matrix $A_M$ such that $A_M=A_N \bigodot M$ where $M$ is a mask matrix with $M_{ij}\sim B(1,p)$ and $\begin{pmatrix}
    1&...&1
\end{pmatrix}\cdot M_{:j}\ge1,\;\forall j\in[1,d]$. We then compute the matrix $A$ and set each $A_i$ in the same way.
\[
A=\begin{pmatrix}
    \frac{2^{-20}n}{\sum_{i=1}^nA^M_{i1}}&\frac{2^{-19}n}{\sum_{i=1}^nA^M_{i2}}&...&\frac{2^{-1}n}{\sum_{i=1}^nA^M_{i20}}&1&...&1&\frac{10n}{\sum_{i=1}^nA^M_{id}} \\ 
    &...& &...& &...& &... \\
    \frac{2^{-20}n}{\sum_{i=1}^nA^M_{i1}}&\frac{2^{-19}n}{\sum_{i=1}^nA^M_{i2}}&...&\frac{2^{-1}n}{\sum_{i=1}^nA^M_{i20}}&1&...&1&\frac{10n}{\sum_{i=1}^nA^M_{id}} 
\end{pmatrix} \bigodot A_M\;.
\]
Through the construction, the smallest eigenvalues of $\nabla^2 f(\xx)$ are clustered around zero, and the largest eigenvalue is 10. Additionally, each $\nabla^2 f_i(\xx)$ is positive semi-definite.

We set the batch size to be 1 and thus we have $f_{i_t}^\star=0$ with interpolation and $\ell_{i_t}^\star=0$ without interpolation.

In addition to the optimizers reported in the main paper, we further evaluate the performance of SLS~\cite{sls}, AdaSLS, SGD, SPS$_{\max}$~\cite{sps} and AdaSVRLS. We define theoretically suggested hyperparameters $c^{\text{scale}}_l:=c_l\rho\sqrt{\gamma_0||\nabla f_{i_0}(\xx_0)||^2}\ge\frac{1}{2}$ for AdaSLS and $c^{\text{scale}}_l:=c_l\rho\sqrt{\gamma_0||\nabla F_{i_0}(\xx_0)||^2}\ge\frac{1}{2}$ for AdaSVRLS. We fix $c^{\text{scale}}_l=1$, $\gamma_{\max}=10$, $\beta=0.8$ and $\rho=1/2$ for both algorithms and for AdaSVRLS, we further use $\mu_F=10$ and $p_t=\frac{1}{0.1t+1}$. For SGD, we use the best learning rate schedules in different scenarios. Specifically, for both interpolation problems, we keep the stepsize constant and for non-interpolation problems, we apply $\cO(1/\sqrt{t})$ and $\cO(1/t)$ decay schedules for convex and strongly-convex problems respectively. We further pick the best stepsize from $\{10^i\}_{i=-4,...,3}$. For SPS$_{\max}$, we use $\gamma_b=10^{-3}$ and we only showcase its performance in non-interpolated settings. We report the results in Figure~\ref{fig:syn_appendix}. We observe that AdaSLS is comparable if no faster than the best tuned vanilla SGD. SPS$_{\max}$ is reduced to the vanilla SGD with constant stepsize. AdaSVRLS provides similar performance to AdaSVRPS but due to the cost of additional function evaluations, it is less competitive than AdaSVRPS.

\begin{figure*}[ht!]
    \centering
    \includegraphics[width=1\textwidth]{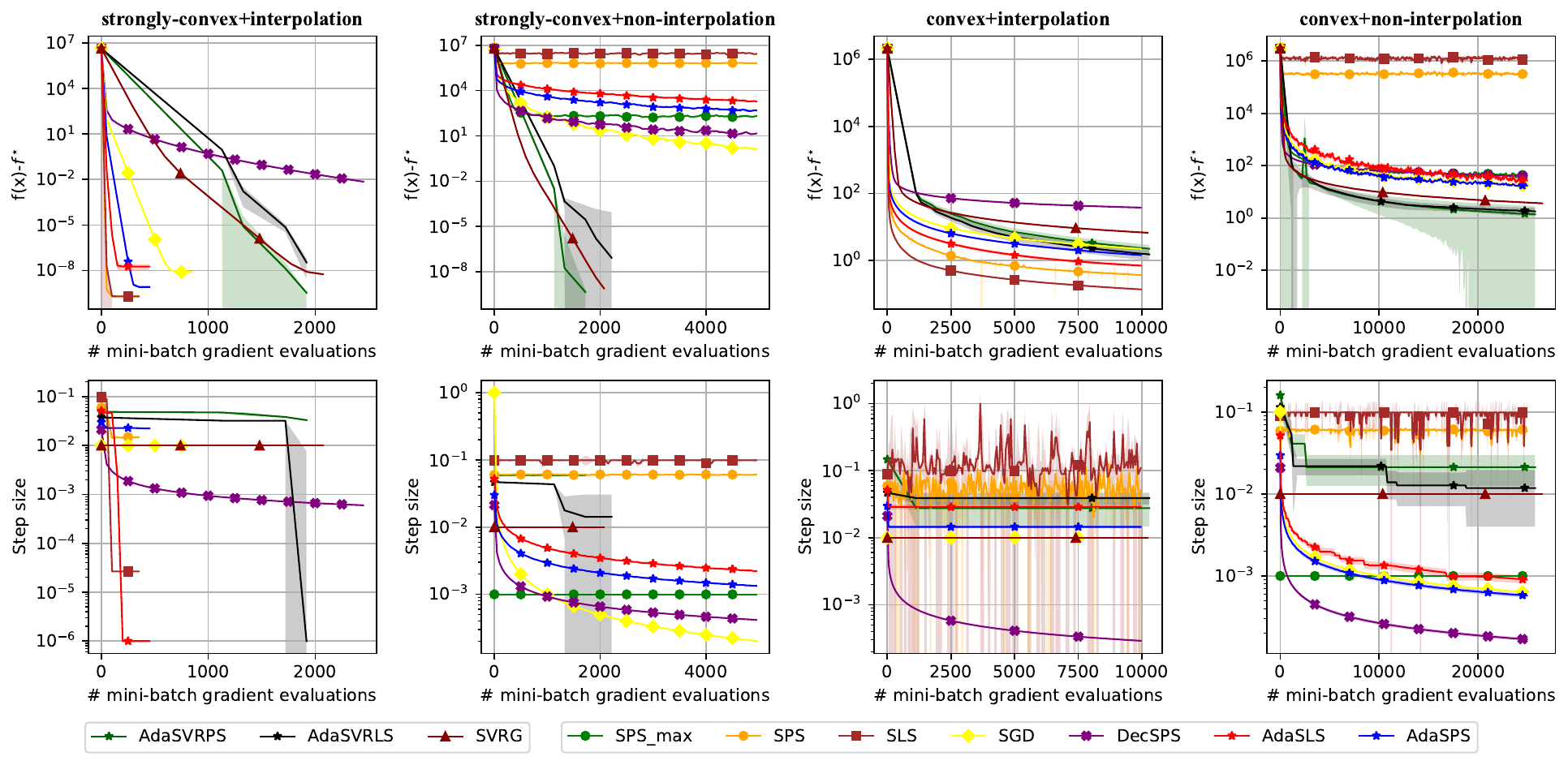}  
    \caption{Comparison of the considered optimizers on synthetic data set with quadratic loss. The left block of the label illustrates the variance-reduced methods and the right represents SGD with different stepsizes. (Repeated 3 times. The solid lines and the shaded area represent the mean and the standard deviation.)}
    \label{fig:syn_appendix}
\end{figure*}

\begin{table*}[ht!]
\resizebox{\textwidth}{!}
{\begin{minipage}{1.2\textwidth}
\centering
\begin{tabular}{@{}c|cccc@{}}
\toprule
\multirow{2}{*}{\textbf{optimizers}} &
\multicolumn{4}{c}{\textbf{hyper-parameters used for synthetic experiments}} \\
\cmidrule(lr){2-5} &
\multicolumn{1}{c}{st-convex+ip} & \multicolumn{1}{c}{st-convex+non-ip} & \multicolumn{1}{c}{convex+ip} & \multicolumn{1}{c}{convex+non-ip} \\
\midrule
AdaSPS & \multicolumn{4}{c}{$c_{p}^{\text{scale}}=1$}\\
AdaSLS & \multicolumn{4}{c}{$c_{l}^{\text{scale}}=1$, $\beta=0.8$, $\rho=0.5$, $\gamma_{\max}=10$}\\
SPS & \multicolumn{4}{c}{$c=0.5$}\\
SPS$_{\max}$ & \multicolumn{4}{c}{$c=0.5$, $\gamma_b=10^{-3}$}\\
SLS & \multicolumn{4}{c}{$\rho=0.1$, $\beta=0.9$, $\gamma_{\max}=10$}\\
DecSPS & \multicolumn{4}{c}{$c_0=1, \gamma_b=10$}\\
SGD & constant, $\eta=10^{-2}$ & $\cO(1/t)$, $\eta=1$& constant, $\eta=10^{-2}$ & $\cO(1/\sqrt{t})$, $\eta=10^{-1}$\\
AdaSVRPS & \multicolumn{4}{c}{$c_{p}^{\text{scale}}=1$, $\mu_F=10$, $p_t=\frac{1}{0.1t+1}$}\\
AdaSVRLS & \multicolumn{4}{c}{$c_{l}^{\text{scale}}=1$, $\beta=0.8$, $\rho=0.5$, $\gamma_{\max}=0.1$, $\mu_F=10$, $p_t=\frac{1}{0.1t+1}$}\\
SVRG & \multicolumn{4}{c}{$\eta=10^{-2}$}\\
\bottomrule
\end{tabular}
\end{minipage}}
\caption{Hyper-parameters of the considered optimizers used in synthetic experiments. st-convex stands for strongly-convex and ip stands for interpolation.}
\label{tab:params:syn}
\end{table*}

\subsection{Binary classification}
\label{appendix:ex:libsvm}
Following the binary classification experiments presented in the main paper, we provide additional experiments for line-search type algorithms. We fix $\gamma_{\max}=10^3$, $\beta=0.8$ and $\rho=1/2$ for AdaSLS and AdaSVRLS. 
We report the best $c_{l}^{\text{scale}}\in\{0.5,1,2\}$, $\mu_F\in\{10^{-4},10^2\}$. In Figure~\ref{fig:libsvm_appendix}, we observe that SLS has the same issues as SPS, AdaSLS shows comparable performance to AdaSPS and AdaSVRLS provides similar performance to the other three variance-reduced methods. The details of the four considered datasets are summarized in Table~\ref{tab:dataset}. The chosen hyper-parameters for each algorithm can be found in Table~\ref{tab:params:libsvm}.

We next investigate the impact of the probability schedule on the convergence behaviours. We pick w8a as the dataset and run AdaSVRPS (Alg.~\ref{AdaVR}) with and without probability decay. Specficially, we set $p_t=B/n$ and $p_t=\frac{1}{0.1t+1}$ to separate the cases. We control the level of the interpolation by using $B=32$ and $B=128$ since $\sigma_{f,128}\le\sigma_{f,32}$. From Figure~\ref{fig:comp_prob}, we observe that decreasing probability schedule is more efficient when the problem is more non-interpolated. This is because for interpolated problems, the frequent computation of the full gradients at the beginning provides no additional convergence benefits. 

\begin{table*}[ht!]
\resizebox{\textwidth}{!}
{\begin{minipage}{1.2\textwidth}
\centering
\begin{tabular}{ccccc@{}}
\toprule
\multicolumn{1}{c}{} &
\multicolumn{1}{c}{\textbf{duke}} & 
\multicolumn{1}{c}{\textbf{rcv1}} & 
\multicolumn{1}{c}{\textbf{ijcnn}} & 
\multicolumn{1}{c}{\textbf{w8a}}  \\
\midrule
n & 44 & 20242  & 49990 & 49749 \\
d & 7129 & 47236 & 22 & 300 \\
B & 1 & 64 & 64 & 128 \\
\bottomrule
\end{tabular}
\end{minipage}}
\caption{Number of datapoints, dimension of features, used batch size of four datasets from LIBSVM~\cite{libsvm}}
\label{tab:dataset}
\end{table*}

\begin{figure*}[ht!]
    \centering
    \includegraphics[width=1\textwidth]{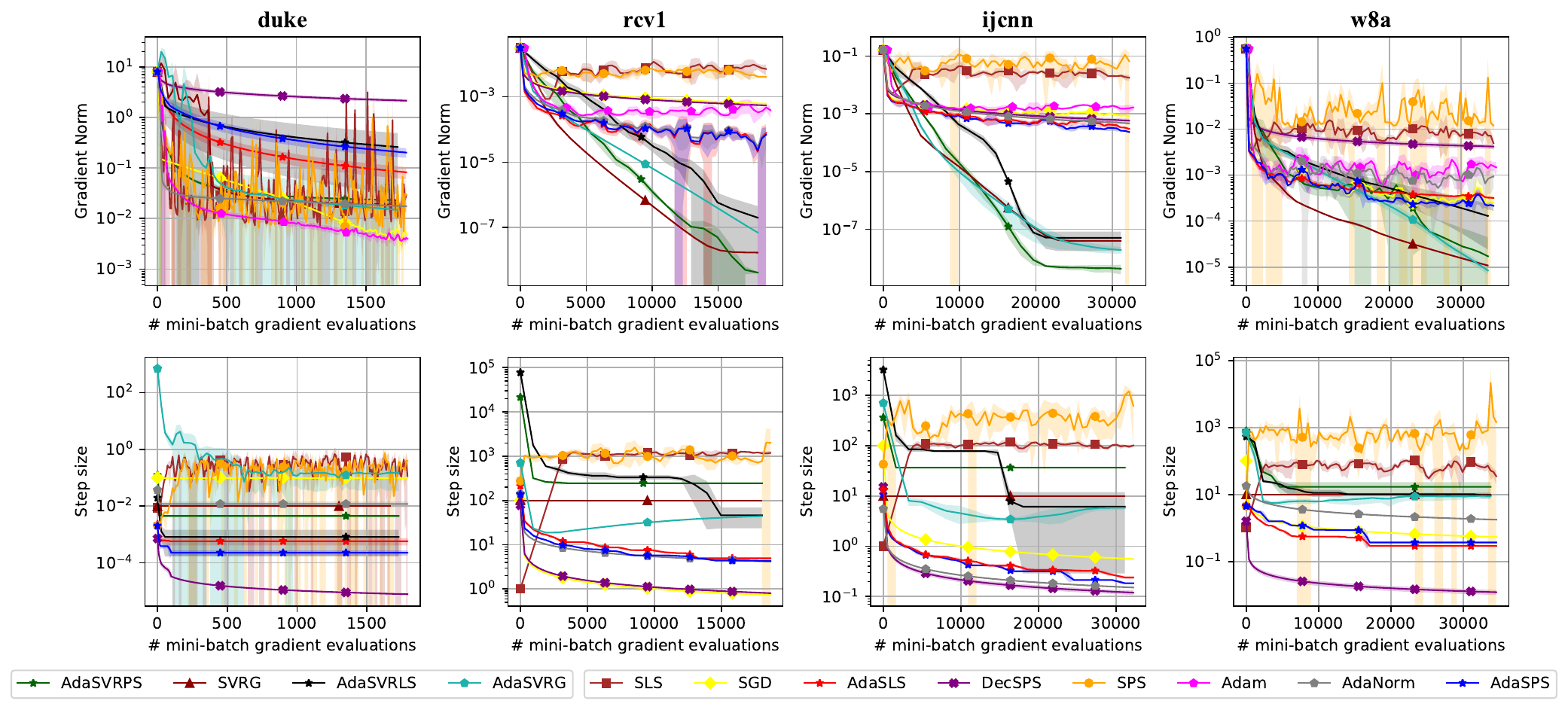}  
    \caption{Comparison of the considered optimizers on four LIBSVM datasets with regularized logistic loss. The left block of the label illustrates the variance-reduced methods and the right represents SGD with different stepsizes. (Repeated 3 times. The solid lines and the shaded area represent the mean and the standard deviation.)}
    \label{fig:libsvm_appendix}
\end{figure*}

\begin{figure*}[ht!]
    \centering
    \includegraphics[width=0.75\textwidth]{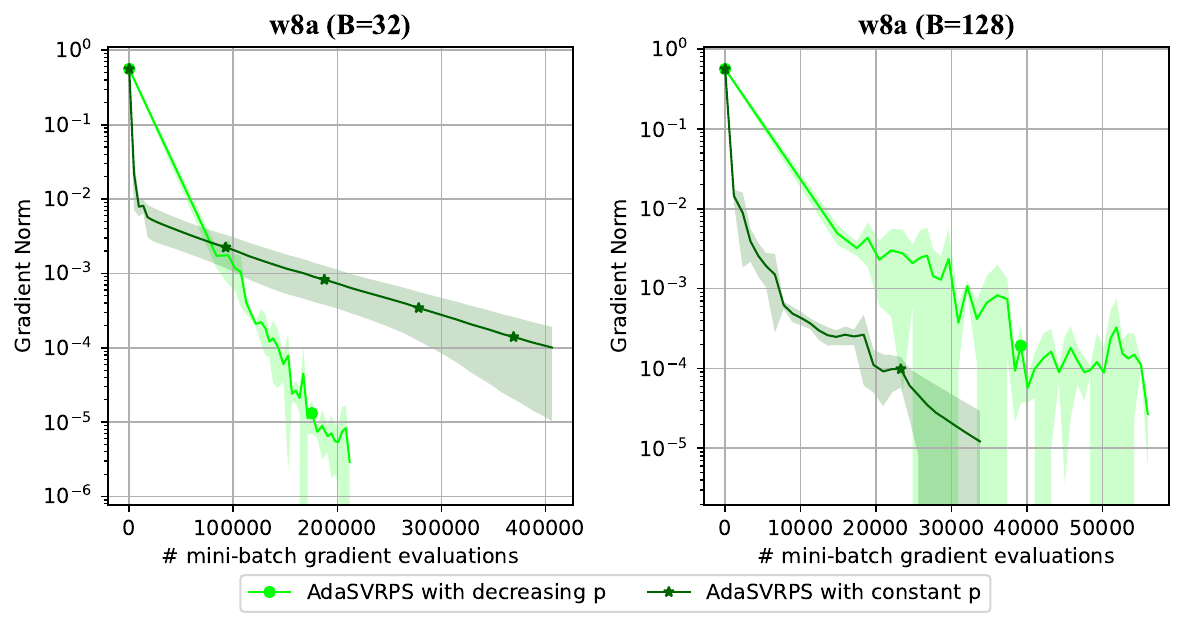}  
    \caption{Comparison of different probability schedules for AdaSVRPS on the w8a dataset with regularized logistic loss. Decreasing probability is more efficient when optimizing highly non-interpolated convex problems (Repeated 3 times. The solid lines and the shaded area represent the mean and the standard deviation.)}
    \label{fig:comp_prob}
\end{figure*}

\begin{table*}[ht!]
\resizebox{\textwidth}{!}
{\begin{minipage}{1.5\textwidth}
\centering
\begin{tabular}{@{}c|ccccc@{}}
\toprule
\multirow{2}{*}{\textbf{optimizers}} &
\multicolumn{5}{c}{\textbf{hyper-parameters used for binary classification tasks}} \\
\cmidrule(lr){2-6} &
\multicolumn{1}{c}{duke} & \multicolumn{1}{c}{rcv1} & \multicolumn{1}{c}{ijcnn} & \multicolumn{1}{c}{w8a} & \multicolumn{1}{c}{for all}\\
\midrule
AdaSPS & $c_{p}^{\text{scale}}=0.5$ & $c_{p}^{\text{scale}}=1$ & $c_{p}^{\text{scale}}=2$ & $c_{p}^{\text{scale}}=0.5$ & \\
AdaSLS & $c_{l}^{\text{scale}}=0.5$, $\gamma_{\max}=10$ & $c_{l}^{\text{scale}}=1$, $\gamma_{\max}=10^3$ & $c_{l}^{\text{scale}}=2$, $\gamma_{\max}=10^3$ & $c_{l}^{\text{scale}}=0.5$, $\gamma_{\max}=10^3$  & $\beta=0.8$, $\rho=0.5$\\
SPS & & & & & $c=0.5$\\
SLS & $\gamma_{\max}=10$& $\gamma_{\max}=10^3$ & $\gamma_{\max}=10^3$ & $\gamma_{\max}=10^3$  &$\beta=0.9$, $\rho=0.1$ \\
DecSPS &$\gamma_b=200$ & $\gamma_b=100$& $\gamma_b=100$& $\gamma_b=100$ &$c_0=1$\\
SGD & constant, $\eta=10^{-1}$ & $\cO(1/\sqrt{t})$, $\eta=100$& $\cO(1/\sqrt{t})$, $\eta=100$ & $\cO(1/\sqrt{t})$, $\eta=100$\\
AdaNorm & $c_g=1$ & $c_g=10$ & $c_g=10$ & $c_g=10$ & $b_0=10^{-10}$\\
Adam & ${\rm lr}=10^{-3}$ & ${\rm lr}=10^{-2}$ & ${\rm lr}=10^{-2}$ & ${\rm lr}=10^{-2}$ & $\beta_1=0.9$, $\beta_2=0.999$\\
AdaSVRPS & $c_{p}^{\text{scale}}=0.5$, $\mu_F=100$ & $c_{p}^{\text{scale}}=1$, $\mu_F=10^{-4}$ & $c_{p}^{\text{scale}}=2$, $\mu_F=10^{-4}$ &  $c_{p}^{\text{scale}}=0.5$, $\mu_F=10^{-4}$ & $p_t=\frac{B}{N}$\\
AdaSVRLS\footnote{$\gamma_{\max}=\frac{1}{\mu_F}$} & $c_{l}^{\text{scale}}=0.5$, $\mu_F=100$ & $c_{l}^{\text{scale}}=1$, $\mu_F=10^{-4}$ & $c_{l}^{\text{scale}}=2$, $\mu_F=10^{-4}$ &  $c_{l}^{\text{scale}}=0.5$, $\mu_F=10^{-4}$ & $\beta=0.8$, $\rho=0.5$, $p_t=\frac{B}{N}$\\
SVRG & $\eta=10^{-2}$ & $\eta=100$ & $\eta=10$ & $\eta=10$ & \\
AdaSVRG & \multicolumn{5}{c}{We use the trick provided in Section 5 from~\cite{ada-svrg}.} \\
\bottomrule
\end{tabular}
\end{minipage}}
\caption{Hyper-parameters of the considered optimizers used in binary classification.}
\label{tab:params:libsvm}
\end{table*}

\section{Deep learning task}
\label{appendix:ex:cifar}
In this section, we provide a heuristic extension of AdaSPS to over-parameterized non-convex optimization tasks. When training modern deep learning models, \citet{warmupstepsize} observe that a cyclic behaviour of the stepsize, i.e., increasing at the beginning and then decreasing up to a constant, can help with fast training and good generalization performance. Since AdaSPS is a non-increasing stepsize, it excludes such a cyclic behaviour. To address this issue, we provide a non-convex version of AdaSPS which incorporates a restart mechanism that allows an increase of the stepsize according to the local curvature. The full algorithm is summarized in Algorithm~\ref{adasps:dl}. In practice, we can set $u=\frac{B}{n}$.   Algorithm~\ref{adasps:dl} updates the stepsize and $c_p$ at the beginning of each epoch and uses AdaSPS~\eqref{AdaSPS} for the rest of the epoch. 

\begin{algorithm}[H]
\begin{algorithmic}[1]
\caption{AdaSPS (DL)}
\Require $\xx_0\in\R^d$, 
$T\in\mathbb{N}^+$, $c_p^{\text{scale}}>0$, update frequency $u\in\mathbb{N}^+$
\State set $\eta_{-1}=+\infty$
\State set $\epsilon=10^{-10}$
\For{$t=0$ to $T-1$}
 \State uniformly sample $i_t\subseteq[n]$
 \State provide a lower bound $\ell_{i_t}^\star\le f_{i_t}^\star$
 \If{$t$ mod $u$ is $0$}
 \State set $c_p=\frac{c_p^{\text{scale}}}{\sqrt{\sum_{s=0}^{t}f_{i_s}(\xx_s)-\ell_{i_s}^\star}}$
 \State set $\eta_t=\frac{f_{i_t}(\xx_t)-\ell_{i_t}^\star}{c_p||\nabla f_{i_t}(\xx_t)||^2}\frac{1}{\sqrt{\sum_{s=0}^t f_{i_s}(\xx_s)-\ell_{i_s}^\star}+\epsilon}$
 \Else 
 \State set $\eta_t=\min\left\{\frac{f_{i_t}(\xx_t)-\ell_{i_t}^\star}{c_p||\nabla f_{i_t}(\xx_t)||^2}\frac{1}{\sqrt{\sum_{s=0}^t f_{i_s}(\xx_s)-\ell_{i_s}^\star}+\epsilon},\eta_{t-1}\right\}$
 \EndIf
 \State $\xx_{t+1}=\xx_t-\eta_t\nabla f_{i_t}(\xx_t)$
\EndFor
\Return $\xx_T$
\label{adasps:dl}
\end{algorithmic}
\end{algorithm}

Following~\cite{sps,sls}, we benchmark the convergence and generalization performance of AdaSPS (DL)~\ref{adasps:dl} for the multi-class classification tasks on CIFAR10~\cite{cifar10} and CIFAR100~\cite{cifar100} datasets using ResNet-34~\cite{resnet}. We compare against SPS~\cite{sps}, Adam~\cite{adam}, AdaGrad~\cite{AdaGrad}, DecSPS~\cite{decsps} and SGD with momentum. We use the smoothing technique and pick $c=0.02$ for SPS as suggested in \cite{sps}. We use the official implementations of Adam, AdaGrad, and SGD with momentum from \href{https://pytorch.org/docs/stable/optim.html}{https://pytorch.org/docs/stable/optim.html}. We choose ${\rm lr}=10^{-3}$, $\beta_1=0.9$ and $\beta_2=0.999$ for Adam. We choose ${\rm lr}=0.01$ for AdaGrad. We choose ${\rm lr}=0.01$ and $\beta=0.9$ for SGD with momentum. Finally, we pick $c_p^{\text{scale}}=0.02$ for Algorithm~\ref{adasps:dl}. In Figure~\ref{fig:cifar}, AdaSPS (DL) shows competitive performance on both datasets. We leave the study of its theoretical properties to future work.

\begin{figure*}[ht!]
\centering
\includegraphics[width=1\textwidth]{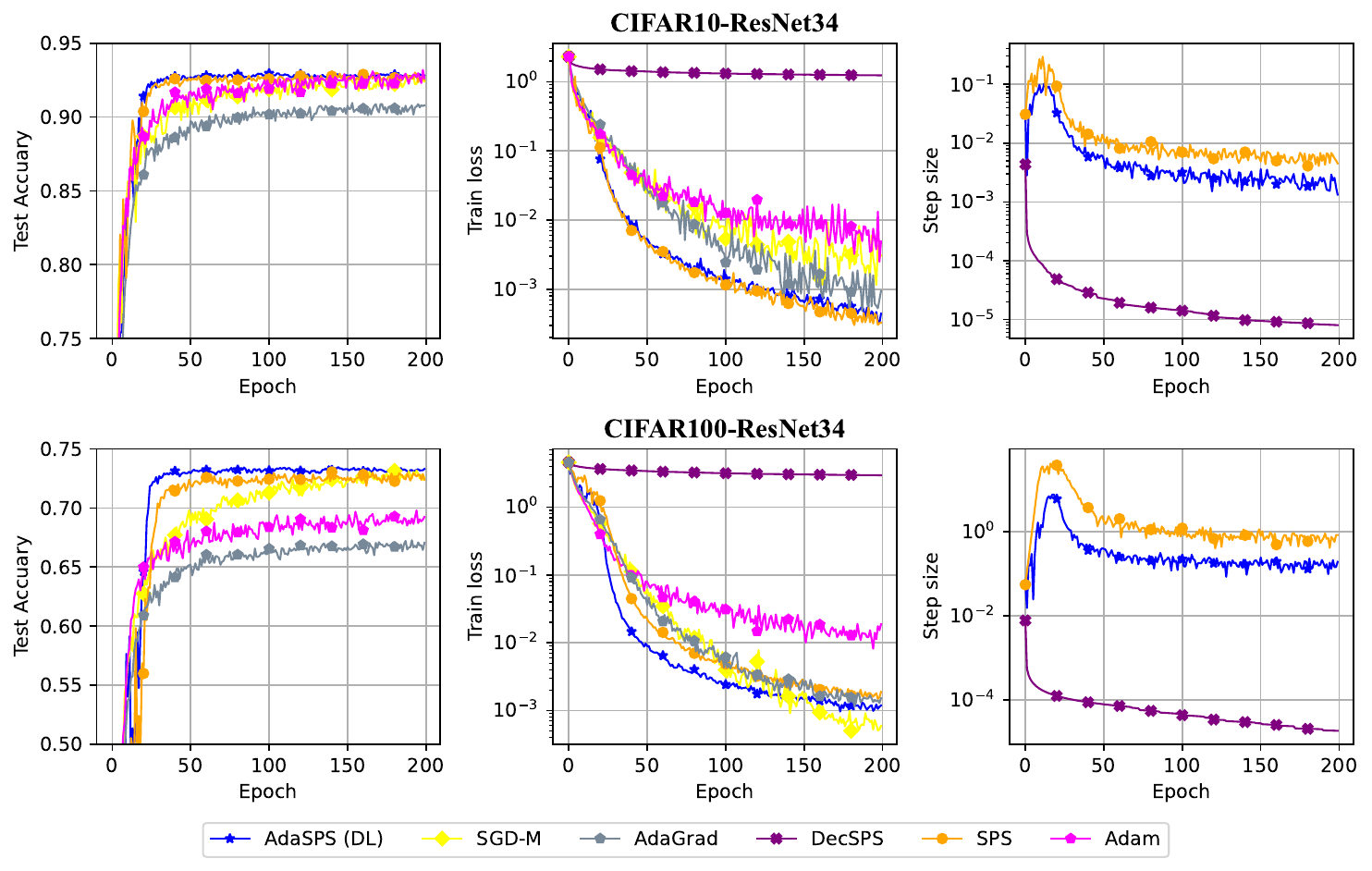}  
\caption{Comparison of the considered optimizers on multi-class classification tasks with CIFAR10 and CIFAR100 datasets using ResNet34 with softmax loss. AdaSPS (DL)~\ref{adasps:dl} and SPS provide remarkable performance on both datasets.}
\label{fig:cifar}
\end{figure*}

\begin{table*}[ht!]
\resizebox{\textwidth}{!}
{\begin{minipage}{1.5\textwidth}
\centering
\begin{tabular}{@{}c|c@{}}
\toprule
\multicolumn{1}{c}{\textbf{optimizers}} &
\multicolumn{1}{c}{\textbf{hyper-parameters used for multi-classification tasks}} \\
\midrule
AdaSPS (DS) & $c_{p}^{\text{scale}}=0.2$ \\
SPS & $c=0.2$ + smoothing technique~\cite{sps}\\
DecSPS & $c_0=1$, $\gamma_b=1000$ \\
SGD-M & ${\rm lr}=0.01$, $\beta=0.9$ \\ 
AdaGrad & ${\rm lr}=0.01$ \\ 
Adam & ${\rm lr}=10^{-3}$, $\beta_1=0.9$, $\beta_2=0.999$ \\
\bottomrule
\end{tabular}
\end{minipage}}
\caption{Hyper-parameters of the considered optimizers used in multi-classification tasks.}
\label{tab:params:cifar}
\end{table*}

\end{document}